\def\nousjournal{1}
\def\nousjmlr{1}
\ifdefined\nousjmlr
\documentclass[twoside,11pt]{article}

\usepackage[preprint]{jmlr2e}
\usepackage{times}
\ShortHeadings{Nous: Certifying Memory Decisions}{Singh}
\firstpageno{1}
\editor{Not assigned}
\else
\documentclass{article}
\ifdefined\nousjournal
\usepackage[margin=1in]{geometry}
\ifdefined\nousijar
\usepackage[numbers,square,sort&compress]{natbib}
\else
\usepackage[authoryear,round]{natbib}
\fi
\usepackage{times}
\else
\usepackage{iclr2027_conference,times}
\fi
\usepackage{amsthm}
\fi
\usepackage{amsmath,amssymb,booktabs,graphicx,microtype}
\usepackage{hyperref}
\hypersetup{hidelinks}
\ifdefined\nousjournal
\hypersetup{pdfauthor={Pranav Singh},
pdftitle={Nous: Learning and Certifying Memory Decisions Before Source Calibration},
pdfsubject={Preprint revision prepared locally. Not uploaded}}
\fi
\ifdefined\nousjmlr
\else
\newtheorem{theorem}{Theorem}
\newtheorem{proposition}[theorem]{Proposition}
\fi
\newcommand{\E}{\mathbb E}
\newcommand{\Prob}{\mathbb P}
\newcommand{\ind}{\mathbf 1}
\newcommand{\KL}{\operatorname{KL}}
\newcommand{\TV}{\operatorname{TV}}
\newcommand{\sgn}{\operatorname{sign}}

\ifdefined\noussupplement
\title{Nous: Learning and Certifying\\Memory Decisions Before Source Calibration\\Supplementary Material}
\else
\title{Nous: Learning and Certifying\\Memory Decisions Before Source Calibration}
\fi
\ifdefined\nousjournal
\ifdefined\nousjmlr
\author{\name Pranav Singh \email 2023mcb1308@iitrpr.ac.in\\
\addr Department of Mathematics\\
Indian Institute of Technology Ropar\\
Bara Phool, Birla Seed Farms\\
Rupnagar, Punjab 140001, India}
\else
\author{Pranav Singh\\Department of Mathematics\\Indian Institute of Technology Ropar\\Rupnagar, Punjab, India\\
\texttt{2023mcb1308@iitrpr.ac.in}}
\fi
\ifdefined\nousijar
\date{October 1, 2026 -- IJAR submission draft}
\else
\ifdefined\nousjmlr
\date{October 2, 2026 -- JMLR research draft}
\else
\date{October 2, 2026 -- journal working draft}
\fi
\fi
\else
\author{Anonymous authors}
\fi
\begin{document}
\raggedbottom
\maketitle
\ifdefined\noussupplement
\else
\begin{abstract}%
Agent memory systems update state decisions from reports whose reliability may
be unknown. Existing analyses of source estimation do not determine when a
policy can be learned or its improvement certified without identifying the
reporting channel. We study these three tasks using the same observed records.
For a specified hidden Markov family with continuous source uncertainty,
decision learning and powered certification have quadratic sample complexity,
whereas fixed-precision source estimation has quartic complexity. We
characterize a sharp identified interval for policy gain under an unknown
shared-background channel and derive finite-sample certificates under bounded
history dependence and conditional copying. Independently trained witness
regions support general history spaces, and disagreement-conditioned auditing
improves power for sparse revisions. For dependent histories,
prediction-count-preserving batches cancel the unknown reporting background
and admit conditional certificates. A MultiWOZ~2.4 evaluation uses
text-processing policies on 1,000 human-written test dialogues with simulated
audits. Balanced batches retain 2.26 percentage points of the full candidate's
7.34 percentage-point mean gain and obtain more positive certificates under
weak audits. These results establish task-specific information requirements
and provide an auditable policy-revision framework for Nous.
\end{abstract}
\ifdefined\nousjournal
\ifdefined\nousjmlr
\else
\noindentKeywords: agent memory, partial identification, noisy labels,
policy evaluation, hidden Markov models.
\fi
\fi

\section{Introduction}
Agent memory systems infer current attributes from reports that may be noisy,
stale or copied. A replacement memory policy requires evidence that its state
decisions improve on those of the deployed policy. Nous represents each entity
attribute by a categorical belief state, called a Dimension, and updates that
state using an observational report, called a Delta \citep{nous}.

Different statistical targets can have different estimation rates
\citep{donoholiu,honguyen}. Hidden Markov model (HMM) analyses establish
difficult source estimation near independence and distinguish prediction
stability from parameter recovery \citep{abraham,gassiat}. These results do
not determine the joint requirements for learning a memory policy and
certifying its improvement from the same observation law. Noisy-label and
partial-identification analyses characterize uncertainty about clean labels
\citep{scott,katzsamuels,maiapolo,guerdan}. A biased audit can nevertheless
favor a less accurate policy when that policy predicts the preferred label
more frequently.

We distinguish decision learning, source calibration and revision
certification. Decision learning selects a current-state prediction policy.
Source calibration estimates numerical source informativeness. Revision
certification tests whether a candidate improves expected accuracy over the
incumbent, which is the currently deployed policy. Powered certification also
requires a specified probability of accepting a beneficial candidate.
Our objective is to determine which of these tasks require source identification
and to construct revisions that can be evaluated without it.

Figure~\ref{fig:receipt-workflow} presents the proposed workflow. We learn a
candidate from earlier observations and fix the comparison before obtaining
fresh audit reports. Population certification uses channels compatible with
the observed report distribution, meaning channels that can produce that
distribution. Witness regions, which are sets selected using history information
alone, bound the unknown reporting background. For an available batch of
revisions, we instead preserve predicted class counts to cancel that background.
Nous records the resulting policy-level certificate separately from observational
updates.

Our contributions are as follows:
\begin{enumerate}
\item We establish matching learning and powered-certification rates on a
continuous-source HMM family, with quadratically greater sample requirements
for source estimation (Theorem~\ref{thm:shared}). All tasks observe the same
full records.
\item We characterize the attainable interval of clean policy gains under
a categorical shared-background channel (Theorem~\ref{thm:sharp}). The sign
criterion does not require histories containing only one clean class.
\item We derive finite-sample certificates under bounded history dependence
and conditional copying (Section~\ref{sec:certificate}). Split-trained witnesses
have power guarantees, and disagreement-conditioned auditing addresses sparse
policy changes.
\item We characterize background cancellation by prediction-count conservation
and construct balanced revision batches (Section~\ref{sec:batch}). Conditional
audits certify the selected batch when its histories are dependent.
\end{enumerate}

We evaluated temporal inference in a mutable-state experiment, memory-dependent
actions in an external simulator, and text-processing policies on human-written
MultiWOZ dialogues. The experiments measured certification frequency alongside
the accuracy gains retained by batch selection. The following sections present
the evidence model, statistical results, experiments and limitations.

\section{Memory Decisions and the Evidence Contract}
Let $K$ be the number of state classes. Let $Z\in\{1,\ldots,K\}$ be an exclusive state, $H$ the information available
when memory is queried, and $Y$ a later report. The incumbent $b(H)$ and
candidate $g(H)$ are fixed before calibration. Define clean and observable gains
\begin{align}
 \Delta&=\E[\ind\{g(H)=Z\}-\ind\{b(H)=Z\}],\label{eq:gain}\\
 \Delta_Y&=\E[\ind\{g(H)=Y\}-\ind\{b(H)=Y\}].\label{eq:proxy}
\end{align}
The target is $\Delta>0$ on a specified deployment population. We call the
resulting policy-level certificate a revision receipt. Sources within a history
may be dependent. Population certificates use independent and identically
distributed (iid) calibration records from the deployment law.
Section~\ref{sec:batch} conditions on a fixed batch and permits dependent
histories under a conditional report-channel contract.

\subsection{Nous Integration}
A Dimension retains the categorical belief state. The native Bayesian updater
generates a baseline candidate. An adapter records the incumbent state and
associates the candidate distribution and policy identifier with that state.
The adapter applies or declines the proposal according to its revision receipt.
It rejects stale states and mismatched policy identifiers. A separate policy
update records this operation because an accuracy certificate does not estimate
the reliability of an observational Delta. The production updater and
persistence format remain unchanged.

\subsection{Information Boundaries}
We fit and select policies using training and development data. Calibration
data evaluate the fixed comparison, and clean test labels remain available only
to the evaluator. A later report can concern a future state, as in the learning
theorem, or a specified current-state version, as in the revision experiments.
These channels have different targets. Exact-channel experiments generate
current-state audits independently of the policy's answer with constant noise
parameters. Misspecification experiments introduce bounded candidate copying.
The tested policies do not receive their future audit reports or clean labels.

Figure~\ref{fig:receipt-workflow} connects this information boundary to the
two receipt targets. Population certification audits an iid sample from the
deployment law. Batch certification audits a preselected set of revisions.
Policy, witness and batch selection precede the fresh certification reports.

\begin{figure}[tbp]
\centering
\includegraphics[width=\linewidth]{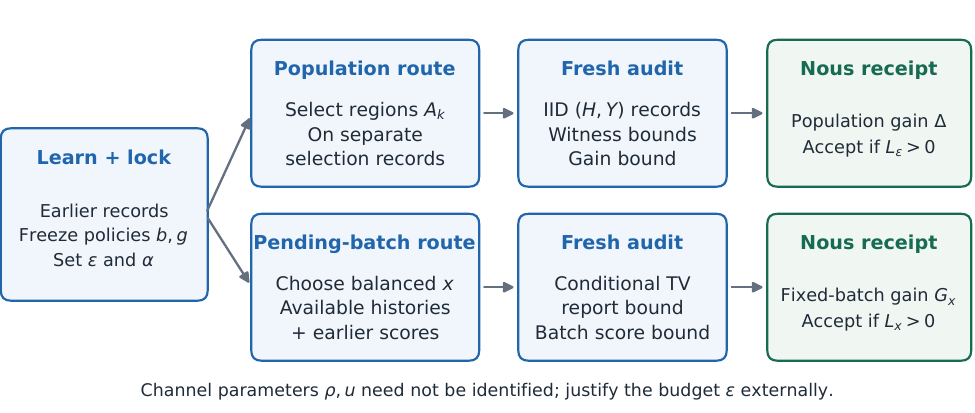}
\caption{Population and fixed-batch revision certification in Nous.
IID denotes independent and identically distributed records.}
\label{fig:receipt-workflow}
\end{figure}

\section{Learning and Certification before Source Calibration}\label{sec:matched}
Fix $0<\ell\leq1/8$. Each observation is the full record
$O=(X_0,X_1,Y)$, where $X_t\in\{-1,+1\}^3$ and $Y=X_{2,1}$.
The target is $Z_1\in\{-1,+1\}$. Records are independent three-step
realizations. Let
$\mathcal M_\ell$ contain stationary balanced binary HMMs with
\begin{equation*}
 \Prob(Z_{t+1}=z'\mid Z_t=z)=(1+\lambda zz')/2,
 \qquad \lambda\in[4\ell/9,\ell].
\end{equation*}
Vector emissions are independent across time conditional on the latent path,
have strictly positive support, and satisfy $f_-(x)=f_+(-x)$. Dependence among
the three simultaneous sources is unrestricted. Anchor informativeness
$a=\E[X_{t,1}\mid Z_t=+1]$ lies in $[1/4,4/5]$. Its known positive orientation
fixes latent-label ambiguity. Both learners know this class and receive the
same $O$, without knowing its parameters.

\subsection{Shared Observation Mechanism}\label{sec:unified}
The central construction uses the same future report for learning and auditing,
without an additional pinned-audit sensor.
Let $s\in\{-1,1\}$, $c\in[1,3/2]$, $\chi(x)=x_1x_2x_3$, and
\begin{equation*}
 \eta_s(x)=\tfrac14x_1+\tfrac38s\chi(x),\quad
 f_z(x)=\tfrac18[1+zc\eta_s(x)],\quad \lambda=\ell/c^2.
\end{equation*}
These worlds belong to $\mathcal M_\ell$. Their source informativeness
is $a=c/4$. All procedures observe the full $O=(X_0,X_1,X_{2,1})$.
\begin{theorem}[Learning and certification precede source calibration]\label{thm:shared}
For fixed failure probabilities below $1/4$ and excess-risk tolerance below $3/128$,
Bayes decision learning on this family requires $\Theta(\ell^{-2})$ records.
Learning $g_s=s\chi(X_1)$ and certifying its improvement over the informative
incumbent $b=X_{1,1}$ using the witness certificate on a fresh split of the
same records requires $O(\ell^{-2})$ records per split. Uniform source
estimation to error below $1/16$ requires $\Theta(\ell^{-4})$ total records,
even if $s$ is given and all splits are pooled.
Any uniformly level-$\alpha$ test for the fixed proposal $g_+$, with power at
least $1-\eta$ in its beneficial world, requires $\Omega(\ell^{-2})$ records when
$\alpha+\eta<1$ are fixed. Hence fixed-candidate certification has matching rates.
\end{theorem}
In this family, $\Delta=c/16$, $m=0$, and $\Delta_Y=\ell/64$. Both witness regions
may be the whole history space. For each decision orientation the endpoint scales
remain calibration-hard. For each scale the orientation must be learned.
Appendix~\ref{app:shared} gives explicit learning/acceptance probabilities,
constants and both source bounds. The rates concern the minimax requirements
of this family. Joint uncertainty about source informativeness and transition
persistence produces the source-estimation lower bound. If persistence is
known, $\lambda$ and $\ell$ determine $c=\sqrt{\ell/\lambda}$.
Table~\ref{tab:shared} summarizes the observable signals and task complexities.

\begin{table}[t]
\centering
\small
\begin{tabular}{lll}
\toprule
Unknown & Observable signal in the same record $O$ & Task complexity\\
\midrule
Orientation $s$ & $\E[\chi(X_1)Y]=3s\ell/32$ & Learn: $\Theta(\ell^{-2})$\\
Policy improvement & $\Delta_Y=\ell/64$ & Certify: $\Theta(\ell^{-2})$\\
Source scale $c$ & $\E[X_{0,1}Y]=\ell^2/(16c^2)$ & Estimate: $\Theta(\ell^{-4})$\\
\bottomrule
\end{tabular}
\caption{Observable signals and sample complexities for the shared-observation family.}

\label{tab:shared}
\end{table}

\subsection{Sufficient Conditions}
Appendix~\ref{app:information-order} gives sufficient conditions on arbitrary
observation spaces. A bounded decision moment and sign-valid audit score have
order-$\ell^r$ margins. Full-law Kullback--Leibler (KL) divergence bounds have orders $\ell^{2r}$ and
$\ell^{2q}$, $q>r$. Learning and powered certification then have
$\Theta(\ell^{-2r})$ complexity, while source estimation needs
$\Omega(\ell^{-2q})$. A stable inverse moment gives the matching upper bound.
The broader class admits a uniform decision guarantee. Its calibration-hard
and decision-hard alternatives can lie in different subfamilies.
\begin{theorem}[Matched decision--calibration separation]\label{thm:matched}
For fixed excess-risk tolerance $0<\epsilon<.05$ and fixed failure probability
$0<\delta<1/4$, the minimax sample complexity of learning a current-state
decision over $\mathcal M_\ell$ is $\Theta(\ell^{-2})$. Uniformly estimating
$a$ with error strictly below $1/8$ and failure probability below $1/4$
requires $\Omega(\ell^{-4})$ records on the same class and experiment.
The decision comparator is the Bayes rule based on both snapshots.
\end{theorem}
The future anchor satisfies $\E[Y\mid Z_1,H]=\lambda aZ_1$. Thus future-label
excess risk is exactly $\lambda a$ times clean-state excess risk. There are
64 possible histories. Empirical future-label majority in each cell gives
\begin{equation}\label{eq:upper}
 R_Z(\hat g)-R_Z(g^*)\leq
 \frac{18}{\ell}\sqrt{\frac{\log(2\cdot2^{64}/\delta)}{2n}}.
\end{equation}
The learner does not estimate $a$ or $\lambda$. Decision-hard full-record laws have
$O(\ell^2)$ KL, whereas parameter-hard laws match the lag-one coefficient
and differ only at order $\ell^2$ in the lag-two coefficient, giving
$O(\ell^4)$ KL. Appendix~\ref{app:matched} gives proofs and simulations.
Theorem~\ref{thm:matched} gives a source lower bound without a matching upper
bound. Section~\ref{sec:related} relates these comparisons to HMM
parameter-estimation results.

\section{Policy-Gain Identification}\label{sec:compatibility}
Deployment reports can have asymmetric noise.
Assume the categorical report channel has the form
\begin{equation}\label{eq:channel}
 C=\rho I+u\mathbf1^{\mathsf T},\qquad
 \rho>0,\quad u\geq0,\quad \mathbf1^{\mathsf T}u=1-\rho.
\end{equation}
With probability $\rho$, the report preserves the state. Otherwise, it draws
from an unknown shared background distribution. This family contains every
positively oriented binary class-conditional channel. Define the clean posterior
$p_k(h)=\Prob(Z=k\mid H=h)$. Set $q_k(h)=\Prob(Y=k\mid H=h)$ and
$m_k=\Prob(g=k)-\Prob(b=k)$. Then
\begin{equation}\label{eq:identity}
 q(h)=\rho p(h)+u,\qquad \Delta_Y=\rho\Delta+m^{\mathsf T}u,
\end{equation}
Prediction-mass changes create an additive bias.
Ignoring observable compatibility gives the sufficient rule
$\Delta_Y>\max_k m_k$. In binary labels $0,1$, this is $\Delta_Y>|\E[g-b]|$.
It is impossible for that rule to accept an entirely one-direction change.

\begin{samepage}
\begin{theorem}[Sharp compatible-channel gain interval]\label{thm:sharp}
Suppose $H$ has finite positive-mass support and $q(h)$ is nonconstant. Write
$\ell_k^Y=\min_h q_k(h)$, so $\sum_k\ell_k^Y<1$. The full compatible family is
\begin{equation}\label{eq:box}
 0\leq u_k\leq\ell_k^Y,\quad \rho=1-\sum_k u_k,\quad
 p(h)=\frac{q(h)-u}{1-\sum_k u_k}.
\end{equation}
The attainable clean gains form the interval between the minimum and maximum of
\begin{equation}\label{eq:functional}
 F(u)=\frac{\Delta_Y-m^{\mathsf T}u}{1-\sum_k u_k}
 \quad\text{over the box in \eqref{eq:box}}.
\end{equation}
In particular, every compatible law has strictly positive gain if and only if
\begin{equation}\label{eq:sharp}
 \Delta_Y>\sum_k(m_k)_+\ell_k^Y.
\end{equation}
\end{theorem}
\end{samepage}

Appendix~\ref{app:certificate} proves the identified interval.

Theorem~\ref{thm:sharp} does not require histories containing only one clean
class. It identifies the policy-gain interval while retaining uncertainty about
the channel. A nonpositive lower endpoint corresponds to an observationally
identical world with nonpositive gain. More records from that observation law
cannot resolve this sign ambiguity.

Figure~\ref{fig:identified-intervals} shows the possible interval positions.
The policy gain may remain positive throughout the compatible family even
when the channel parameters and numerical gain remain unidentified.
An interval crossing zero leaves the sign unresolved.

\begin{figure}[tbp]
\centering
\includegraphics[width=\linewidth]{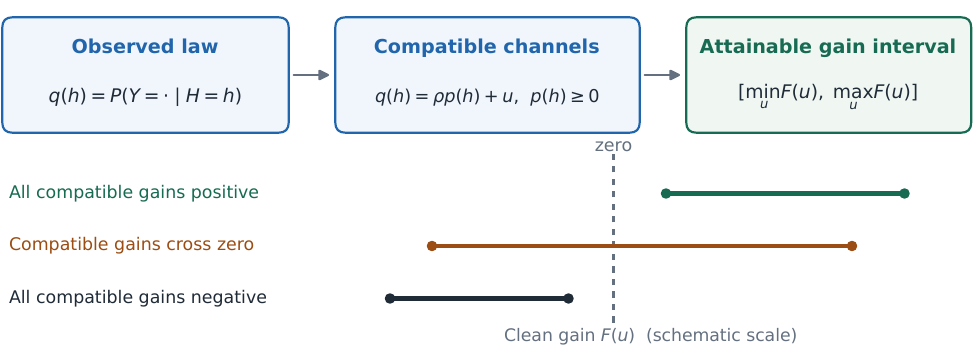}
\caption{Schematic identified intervals for clean policy gain.}

\label{fig:identified-intervals}
\end{figure}

\section{Finite-Sample Certification and Repeated Releases}\label{sec:certificate}
Estimating every conditional report probability can be expensive. Choose,
before calibration, one history-only witness region $A_k$ per class. Regions
can overlap and can omit histories. They do not require pure classes. Since
$u_k\leq\Prob(Y=k\mid H\in A_k)$, an upper bound on a regional report
frequency is also an upper bound on the unknown background mass.

For $n$ iid records compute an upper Clopper--Pearson bound $U_k$ for that
regional frequency with one-sided error $\alpha/(2K)$. Set $U_k=1$ for an
empty region. Let $\widehat\Delta_Y$ and $\hat m$ be the empirical quantities
in \eqref{eq:proxy} and \eqref{eq:identity}, and set
\begin{equation}\label{eq:certificate}
 L=\widehat\Delta_Y-\sum_k(\hat m_k)_+U_k
       -\sqrt{\frac{8\log(2/\alpha)}{n}}.
\end{equation}

\begin{theorem}[Witness certificate]\label{thm:witness}
Under \eqref{eq:channel}, fixed policies and history-only regions, and iid
calibration from the deployment law, $\Prob\{L>0\}\leq\alpha$ whenever
$\Delta\leq0$. More generally $L\leq\rho\Delta$ with probability at least
$1-\alpha$. When $L>0$, it is also a positive lower bound on $\Delta$.
The same records may be used for witness bounds and gain estimation.
\end{theorem}

Appendix~\ref{app:certificate} proves the finite-sample bound.

For a finite history alphabet, independent training selects minimum-frequency
cells. Increasing the training and calibration counts recovers the sharp
condition whenever its margin is strictly positive.

\subsection{Repeated Release Control}
At release $t$ spend $\alpha_t=\alpha/[t(t+1)]$. Policies and witnesses may
depend on past releases but must precede that release's fresh iid calibration.
Conditional validity and $\sum_t\alpha_t=\alpha$ bound the probability of
ever accepting a non-improving release by $\alpha$. Each comparison uses fresh
calibration data and its assigned error budget.

\subsection{Bounded Channel Misspecification}
Let total variation (TV) measure deviation from the reference channel.
The actual channel satisfies
\begin{equation*}
\mathop{\mathrm{ess\,sup}}_{h,z}\mathrm{TV}(C_h(\cdot\mid z),
\rho e_z+u)\leq\epsilon
\end{equation*}
for one $\rho>0$, $u\geq0$,
$\sum_k u_k=1-\rho$. Set $V_k=\min(1,U_k+\epsilon)$ and
$\hat d=n^{-1}\sum_i\ind\{g_i\ne b_i\}$.
\begin{theorem}[Bounded-misspecification certificate]\label{thm:tv}
For policies and history-only witnesses fixed independently of iid calibration,
\begin{equation*}
 L_\epsilon=\hat\Delta_Y-\sum_k(\hat m_k)_+V_k-2\epsilon\hat d
 -(4+2\epsilon)\sqrt{\frac{\log(2/\alpha)}{2n}}
\end{equation*}
satisfies $\Prob(L_\epsilon\leq\rho\Delta)\geq1-\alpha$.
Accepting only when $L_\epsilon>0$ has level-$\alpha$ error for $\Delta\leq0$.
On the same event an accepted $L_\epsilon$ lower-bounds $\Delta$.
\end{theorem}
The proof and local tightness examples are in Appendix~\ref{app:tv}.
A conditional arbitrary-copy budget $\beta$ on top of
base deviation $\gamma$ gives $\epsilon=\gamma+\beta-\gamma\beta$.
We specify these budgets independently of the calibration reports.
History information alone determines the witnesses.
Section~\ref{sec:witness-study} evaluates channel sensitivity and learned witnesses.
Appendix~\ref{app:power} gives an acceptance condition for split-trained
score-threshold regions on arbitrary history spaces.
The condition separates region probability, approximation error and sample size.
Conditional validity composes candidate learning, independent auditing and
policy-level Nous receipts. The parameters $\pi$ and $\xi$ describe sufficient
power conditions.

\subsection{Sparse Policy Changes}
The global Hoeffding radius ignores how rarely the policies differ.
Theorem~\ref{thm:disagreement} instead conditions the paired score on the
history-only event $g\ne b$, while retaining full-sample witness coverage.
For disagreement probability $d$ and a positive conditional margin $\mu$,
Proposition~\ref{prop:sparsepower} gives sufficient sample size
$O((d\mu^2)^{-1})$ when conditional prediction masses are balanced, with
a matching identity-channel lower bound. A binomial lower bound
on $d$ converts conditional evidence into an unconditional positive-gain
receipt. Conditioning reduces finite-sample dilution while preserving the
population identified set.

\section{Constructing Certifiable Revision Batches}\label{sec:batch}
The preceding certificates evaluate a fixed policy on an iid population sample.
A complementary operation is to construct and audit a particular pending batch
of memory revisions. This allows dependent histories, such as multiple turns
from one dialogue, while changing the estimand to the audited batch itself.

Fix histories and truths $(H_i,z_i)_{i=1}^N$, predictions $g_i,b_i$, and a
selection $x_i\in\{0,1\}$ fixed before fresh reports, with $x_i=0$ if $g_i=b_i$.
The vector $e_k$ is the coordinate vector for class $k$.
Write $M=\sum_i x_i$, $a=\sum_i x_i(e_{g_i}-e_{b_i})$, and
\begin{equation*}
 G_x=\frac1N\sum_i x_i(\ind\{g_i=z_i\}-\ind\{b_i=z_i\}).
\end{equation*}
Under a common channel
$C_0=\rho I+u\mathbf1^\top$, the expected sum of audit contrasts is
$\rho N G_x+u^\top a$. Consequently, preserving predicted class counts
($a=0$) cancels the unidentified reporting background exactly.
Figure~\ref{fig:balanced-cycle} gives a three-class example: retaining a cycle
conserves each predicted class count, whereas an additional one-direction
proposal creates a nonzero background term.

\begin{figure}[htbp]
\centering
\includegraphics[width=\linewidth]{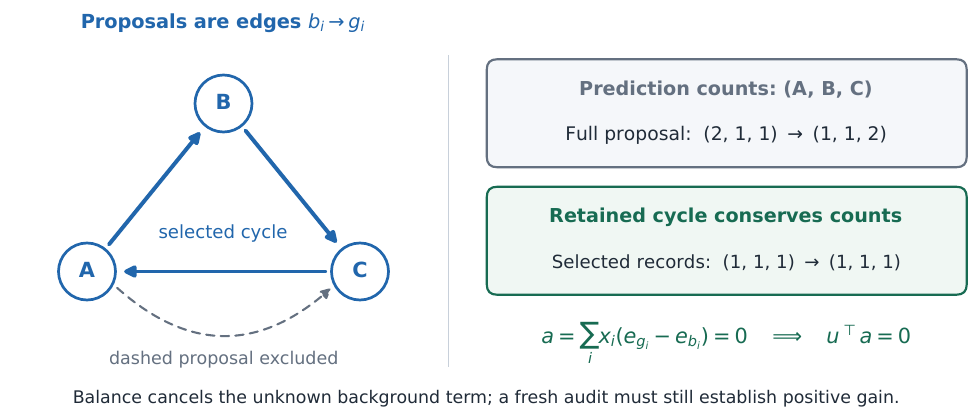}
\caption{A prediction-count-preserving revision cycle.}
\label{fig:balanced-cycle}
\end{figure}

\begin{proposition}[Conservation and construction]\label{prop:circulation}
For any fixed $0<\rho<1$, the additive term $u^\top a$ vanishes for every
$u\geq0$ with $\sum_k u_k=1-\rho$ if and only if $a=0$.
Represent proposal $i$ by a directed edge $b_i\to g_i$. A nonempty balanced
subset exists exactly when this multigraph contains a directed cycle.
Maximizing $\sum_i x_i s_i$ over balanced subsets is a maximum-score
circulation with an integral linear-program optimum. If $s_i$ equals the
expected audit contrast under the exact channel, its optimizer also maximizes
$G_x$ over balanced subsets without knowing $\rho$ or $u$.
\end{proposition}

Conservation preserves the incumbent's prediction histogram without requiring
a uniform histogram or known true class proportions.
The maximum-score circulation uses classical integrality \citep{schrijver}.
Count-constrained label assignments also appear in label-proportion learning
\citep{matsuo}. We use the constraint to cancel the reporting background.
For scores approximating the exact audit contrasts $t_i$, clean regret relative
to the best balanced subset is at most $\sum_i|s_i-t_i|/(\rho N)$.
Appendix~\ref{app:batch} proves this bound.

\begin{samepage}
\begin{theorem}[Fixed-batch receipt]\label{thm:batch}
Condition on the batch histories, truths and selection, with $a=0$ and $M>0$.
Given all preceding audit reports, suppose each selected report's conditional
kernel is within TV $\epsilon$ of the same positive channel $C_0$.
For $W_i=\ind\{g_i=Y_i\}-\ind\{b_i=Y_i\}$ on selected records, set
\begin{equation*}
 L_x=\frac1M\sum_i x_iW_i-2\epsilon
                    -\sqrt{\frac{2\log(1/\alpha)}M}.
\end{equation*}
Accept if and only if $L_x>0$. With probability at least $1-\alpha$, acceptance implies
$0<(M/N)L_x\leq G_x$. For $M=0$ decline. Histories and truths may be dependent and
nonidentically distributed. No witness regions are required.
\end{theorem}
\end{samepage}

The proof applies conditional bounded-score concentration after background
cancellation. The implementation uses KL--Chernoff inversion, which gives a
lower endpoint at least as large as the displayed Hoeffding endpoint.
All primary comparators use this confidence-bound construction.
For an unbalanced subset, subtracting $\max_k a_k/M$ gives a conservative
comparison because $u^\top a/M\leq(1-\rho)\max_k a_k/M$.
The dialogue experiment compares acceptance frequency with retained clean gain.

\section{Experiments}\label{sec:experiments}
We evaluated inference accuracy, revision certification and action outcomes.
Table~\ref{tab:protocols} summarizes the data splits. We selected policies using
training and development data and fixed each comparison before calibration.
Controlled studies used generated or executable state labels. The dialogue
study used published human annotations.

\begin{table}[t]
\centering\small
\begin{tabular}{lrrrr}
\toprule
Study & Training & Development & Calibration & Test\\
\midrule
Mutable state & 4,000 & 1,000 & 5,000 & 5,000\\
MiniGrid & -- & -- & 5,000 & 1,000\\
Learned witnesses & 4,000 & 4,000 & 8,000 & 5,000\\
MultiWOZ & 8,438 & 1,000 & -- & 1,000\\
\bottomrule
\end{tabular}
\caption{Data splits per run. Counts denote histories for mutable state and
learned witnesses, episodes for MiniGrid, and dialogues for MultiWOZ.
The learned-witness development column denotes region-selection data.}
\label{tab:protocols}
\end{table}

\subsection{Mutable-State Evaluation}
We generated independent, copied and delayed-copy reports from a controlled
HMM. A versioned state store executed commands and rejected stale versions.
We checked the withheld current state by replaying the command sequence.
All policies received the same report snapshots, arrival times and source
identifiers. The candidate set included native and root-capped Nous, temporal
predictors and HMM inference baselines. Root-capped Nous used each recorded
source root once per snapshot. Arrival last-write-wins (LWW) retained the
latest arriving report.

Table~\ref{tab:ledger-accuracy} compares inference accuracy across reporting
conditions. Temporal inference improved on arrival LWW. In the copy columns,
the joint-emission HMM achieved greater accuracy than the selected policy.
We selected the candidate using development next-report accuracy and evaluated
the fixed policy with an independent current-state audit. The unrestricted-channel
certificate accepted the selected candidate and rejected a harmful constant
proposal. Table~\ref{tab:acceptance-summary} gives the acceptance counts.
Table~\ref{tab:ledger-gains} distinguishes aggregate gain from harmful and
beneficial individual revisions. Supplementary Appendix~\ref{app:ledger}
provides the command protocol and inference settings.

\subsection{Memory-Dependent Action Evaluation}
We evaluated MiniGrid Memory environments with a noisy-report interface
\citep{minigrid}. The task required remembering an initial object cue and
choosing its matching branch after traversing a corridor. An observation-based
controller handled routing. Memory policies received noisy reports and source
identifiers, and their cue predictions determined the terminal action.
We executed both terminal choices in cloned simulator states to obtain paired
action outcomes.

We compared root-capped Nous with arrival LWW and a constant incumbent.
The witness for each class contained histories in which the first source
reported another class. Table~\ref{tab:minigrid} reports task success.
Table~\ref{tab:acceptance-summary} compares the unrestricted and compatible-channel
certificates. The compatible-channel certificate accepted improvements over
both incumbents, whereas the unrestricted certificate accepted none.
The single-source baseline achieved greater success than root-capped Nous in
each environment column. Supplementary Appendix~\ref{app:minigrid} provides
the paired gains and action-audit procedure.

\begin{table}[t]
\centering\small
\begin{tabular}{lrrr}
\toprule
Method & Memory S7 & Memory S13 & S17 random \\
\midrule
Arrival LWW & 61.67 $\pm$ 1.63 & 60.97 $\pm$ 0.81 & 61.17 $\pm$ 1.50 \\
Constant zero & 50.83 $\pm$ 0.85 & 49.93 $\pm$ 1.27 & 49.70 $\pm$ 2.51 \\
Native Nous / majority & 61.23 $\pm$ 1.00 & 60.70 $\pm$ 0.26 & 61.03 $\pm$ 1.10 \\
Root-capped Nous & 75.50 $\pm$ 1.32 & 74.17 $\pm$ 0.72 & 73.63 $\pm$ 1.42 \\
Single source 0 & 84.40 $\pm$ 0.75 & 85.23 $\pm$ 1.12 & 84.83 $\pm$ 0.78 \\
Old gate vs. LWW & 61.67 $\pm$ 1.63 & 60.97 $\pm$ 0.81 & 61.17 $\pm$ 1.50 \\
Compatible gate vs. LWW & 71.60 $\pm$ 7.69 & 65.33 $\pm$ 7.28 & 65.67 $\pm$ 7.26 \\
Compatible gate vs. constant & 75.50 $\pm$ 1.32 & 74.17 $\pm$ 0.72 & 73.63 $\pm$ 1.42 \\
\bottomrule
\end{tabular}

\caption{MiniGrid task success (\%). Values are means with sample standard
deviations across three runs. LWW denotes last-write-wins.}
\label{tab:minigrid}
\end{table}

\subsection{Categorical and Repeated-Release Controls}
We tested beneficial, null and harmful fixed policies across categorical
state spaces, report delays and calibration sizes. We also evaluated fresh
repeated releases with the error-spending schedule in Section~\ref{sec:certificate}.
Table~\ref{tab:control-summary} reports the negative findings from these controls.

The phase study varied report signal strength and background bias.
Figure~\ref{fig:phase} shows acceptance frequency alongside the population
identification boundary. The zero-acceptance cells include both nonidentifiable
comparisons and identifiable comparisons with insufficient finite-sample power.
The unrestricted certificate accepted no changes throughout the grid.
Supplementary Appendix~\ref{app:controls} specifies the generating laws.

\begin{figure}[t]
\centering
\includegraphics[width=\linewidth]{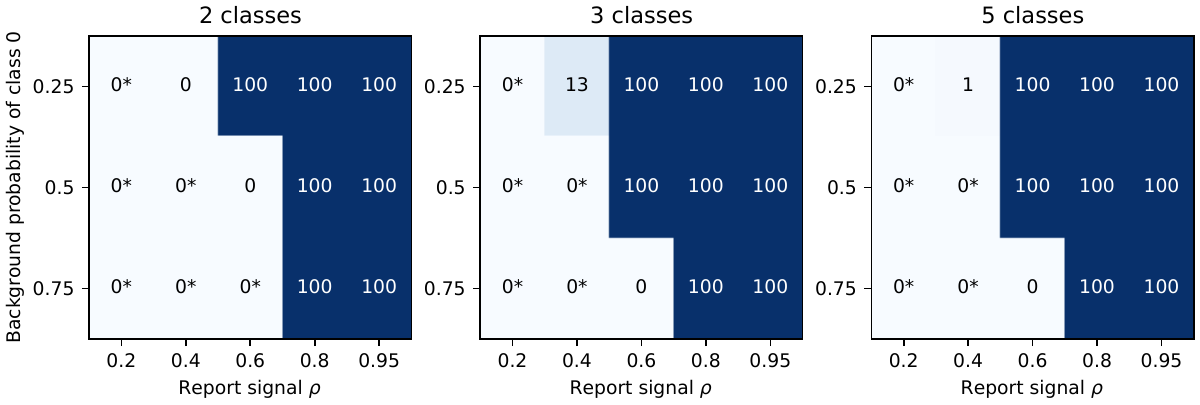}
\caption{Compatible-channel acceptance frequency (\%) by report signal and
background bias. Stars denote nonpositive identified lower endpoints.}
\label{fig:phase}
\end{figure}

\subsection{Channel Sensitivity and Learned Witnesses}\label{sec:witness-study}
We evaluated the largest deviation budget at which each MiniGrid comparison
remained accepted:
\begin{equation*}
\epsilon_{\max}=\sup\{\epsilon\in[0,1]:L_\epsilon>0\}.
\end{equation*}
We used calibration data alone. Supplementary Figure~\ref{fig:sensitivity}
shows how the lower bound decreases with the assumed deviation.
An accepted comparison at zero budget can tolerate only limited misspecification.
Table~\ref{tab:sensitivity-ranges} records the reported tolerance ranges.

We also learned witness regions from continuous histories.
Design data fitted a linear report predictor and candidate policy.
Independent selection data chose among score-threshold regions.
Fresh calibration data evaluated the fixed candidate and witnesses.
Table~\ref{tab:learned} compares learned regions with fixed-coordinate and
whole-history regions. Learned witnesses increased acceptance at the intermediate
deviation setting. All methods declined at the largest setting despite positive
clean gains. Supplementary Appendix~\ref{app:witness-study} provides the
construction and replay procedure.

\begin{table}[t]
\centering\small
\begin{tabular}{rrrrr}
\toprule
TV budget & Learned & Fixed coordinate & Whole history & Test gain (pp) \\
\midrule
0.00 & 4/4 & 4/4 & 0/4 & 28.82 \\
0.03 & 4/4 & 0/4 & 0/4 & 29.39 \\
0.10 & 0/4 & 0/4 & 0/4 & 27.84 \\
\bottomrule
\end{tabular}

\caption{Learned-witness acceptance counts and mean clean test gain
(percentage points). TV denotes total variation and pp denotes percentage points.}
\label{tab:learned}
\end{table}

\subsection{Sparse-Revision Evaluation}
We varied class count, calibration size, disagreement probability, conditional
clean gain and copying probability using a prespecified grid.
The disagreement-conditioned certificate accepted more beneficial proposals
than the global certificate. Neither method accepted null or harmful proposals.
Table~\ref{tab:acceptance-summary} reports the aggregate counts.
The conditioned-bound column of Supplementary Table~\ref{tab:sparse} shows
the sample-size dependence at rare disagreement.
Both rules had zero acceptance in the smallest rare-disagreement setting.
This trend agrees with the dilution addressed by conditional auditing.

\subsection{Human-Written Dialogue Evaluation}\label{sec:dialogue}
MultiWOZ~2.4 contains task-oriented human conversations with corrected
development and test state annotations \citep{multiwoz,multiwoz24}.
We evaluated seven categorical slots. Term frequency--inverse document
frequency (TF--IDF) features and logistic-loss classifiers produced local and
full-history reports. Features contained only dialogue text available at the
query time. Supplementary Appendix~\ref{app:multiwoz} specifies the ontology,
feature construction and implementation settings.

Native Nous, last-report, frequency and recency-weighted policies shared the
local extractor. The full-history classifier used a larger contextual view.
Development accuracy selected the full-history candidate for each slot.
Table~\ref{tab:dialogue} compares its accuracy and revision gains with native Nous.
Balancing retained only part of the candidate's gain.
The state-change columns of Supplementary Table~\ref{tab:dialogue-slices}
show the difference between all-turn accuracy and changing-state performance.
The oracle-extraction results in Supplementary Table~\ref{tab:dialogue-oracle}
indicate that extraction and context account for much of the accuracy gap.

\begin{table}[t]
\centering\small
\begin{tabular}{lrrrrr}
\toprule
Slot & Nous & History & Full gain & Balanced gain & Revisions kept\\
\midrule
Hotel area & 84.56 & 89.64 & 5.07 & 2.24 & 294/737\\
Hotel price & 83.19 & 89.91 & 6.71 & 1.87 & 292/776\\
Hotel parking & 88.28 & 95.40 & 7.12 & 1.65 & 227/725\\
Hotel internet & 87.53 & 95.01 & 7.47 & 1.30 & 164/707\\
Restaurant area & 81.77 & 90.75 & 8.98 & 3.21 & 460/1108\\
Restaurant price & 82.55 & 92.45 & 9.89 & 3.38 & 426/1089\\
Attraction area & 84.79 & 90.88 & 6.09 & 2.16 & 373/845\\
\bottomrule
\end{tabular}

\caption{MultiWOZ test accuracy (\%), batch gain (percentage points), and
retained revision counts. History denotes the full-history classifier.}
\label{tab:dialogue}
\end{table}

Each slot's available test turns formed a fixed batch with dependent dialogue
histories. We selected balanced subsets before generating fresh reports.
The audit simulation varied reference signal strength and candidate copying.
All comparators used the same underlying reports and confidence-bound construction.
The scored and count-only balanced methods committed different subsets from
the full corrected proposal. We therefore evaluated deployed gain as well as
acceptance frequency.

Table~\ref{tab:acceptance-summary} reports forward and harmful reverse
acceptances. The corrected methods rejected every harmful reverse proposal.
The naive report-agreement test accepted harmful proposals.
Figure~\ref{fig:batch-tradeoff} shows the utility trade-off.
In its low-signal region, balancing increased deployed gain through more frequent
certification. In its high-signal region, the full corrected proposal retained
greater deployed gain. Supplementary Table~\ref{tab:dialogue-audit-grid}
reports the complete channel grid.

\begin{figure}[t]
\centering
\includegraphics[width=\linewidth]{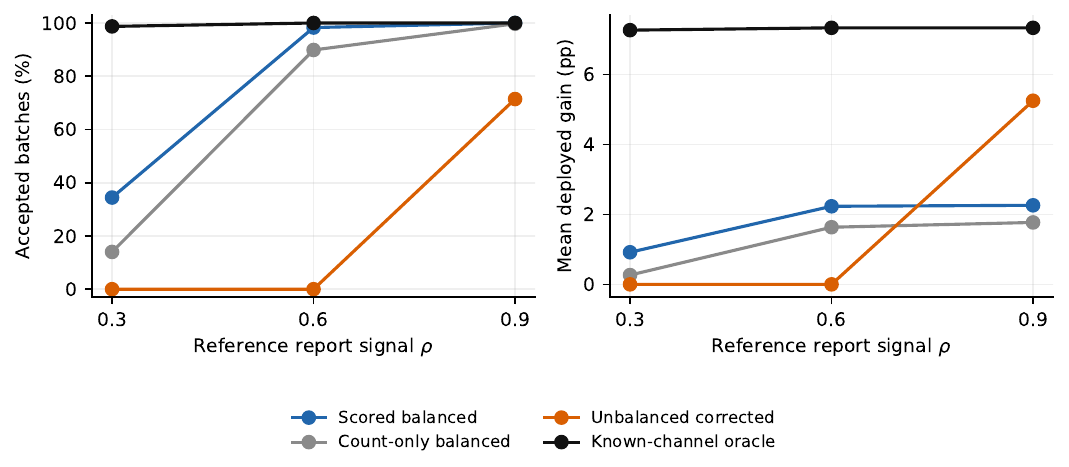}
\caption{Acceptance frequency (\%, left) and mean deployed batch gain
(percentage points, right) under no-copy audits.}
\label{fig:batch-tradeoff}
\end{figure}

\begin{table}[t]
\centering\small
\begin{tabular}{llrr}
\toprule
Study & Method or comparison & Beneficial accepts & Harmful accepts\\
\midrule
Mutable state & Unrestricted certificate & 9/9 & 0/9\\
Mutable state & Report agreement & -- & 9/9\\
MiniGrid & Compatible vs. constant & 9/9 & --\\
MiniGrid & Compatible vs. LWW & 4/9 & --\\
Sparse revisions & Global certificate & 1,558/3,600 & 0/7,200\\
Sparse revisions & Conditioned certificate & 2,720/3,600 & 0/7,200\\
MultiWOZ & Scored balanced & 4,196/6,300 & 0/6,300\\
MultiWOZ & Count-only balanced & 3,434/6,300 & 0/6,300\\
MultiWOZ & Full corrected & 1,076/6,300 & 0/6,300\\
MultiWOZ & Known-channel oracle & 5,860/6,300 & 0/6,300\\
MultiWOZ & Naive report agreement & -- & 1,959/6,300\\
\bottomrule
\end{tabular}
\caption{Revision acceptances. LWW denotes last-write-wins.
The sparse-revision harmful column includes null proposals.
Dashes denote unreported or inapplicable comparisons.}
\label{tab:acceptance-summary}
\end{table}

\begin{table}[t]
\centering\small
\begin{tabular}{lr}
\toprule
Control outcome & Reported count or bound\\
\midrule
Initial beneficial acceptances & 0/1,800\\
Initial null/harmful acceptances & 0/3,600\\
Repeated-release false acceptances & 0/200\\
Repeated-release upper bound (\%) & 1.83\\
Phase cells with complete acceptance & 23\\
Phase cells with partial acceptance & 2\\
Phase cells with zero acceptance & 20\\
Zero-acceptance nonidentified cells & 16\\
Zero-acceptance underpowered cells & 4\\
\bottomrule
\end{tabular}
\caption{Categorical controls, repeated releases and phase-study outcomes.}
\label{tab:control-summary}
\end{table}

\begin{table}[t]
\centering\small
\begin{tabular}{lrr}
\toprule
Accepted comparison & Minimum tolerance (\% TV) & Maximum tolerance (\% TV)\\
\midrule
Constant incumbent & 3.03 & 5.07\\
LWW incumbent & 0.18 & 1.37\\
\bottomrule
\end{tabular}
\caption{Reported misspecification-tolerance ranges among accepted MiniGrid
comparisons. TV denotes total variation and LWW denotes last-write-wins.}
\label{tab:sensitivity-ranges}
\end{table}

\section{Related Work}\label{sec:related}
We situate the statistical results within task-specific estimation, noisy-label evaluation,
constrained prediction and agent memory.
\subsection{Task-Specific Statistical Rates}
Geometric analyses connect functionals to their estimation rates
\citep{donoholiu}, and mixture identifiability controls parameter rates
\citep{honguyen}. \citet{abraham} establish quartic dependence on proximity
to rank deficiency for HMM emission estimation. \citet{gassiat} contrast stable
prediction with unstable parameter recovery near independence.
Our shared-observation result compares learning, powered certification and
source estimation on one continuous-source family.
The broader-class result gives decision rates and a source lower bound using
different least-favorable subfamilies.

\subsection{Noisy Labels and Partial Identification}
\citet{scott} characterize the feasible binary contamination family.
\citet{katzsamuels} study multiclass mutual contamination under irreducibility
and mixing assumptions. Noisy posterior extrema also determine decision
thresholds and classification rates \citep{menon,reeve}.
Our identified interval targets fixed-policy gain while retaining uncertainty
about the reporting channel. It does not require pure-class histories.
Proxy-label evaluation has established sensitivity and partial-identification
analyses \citep{fogliato,maiapolo}. \citet{su} study classifier testing and
training with noisy or multiple labels given the conditional label law and
class prior. \citet{guerdan} exploit cancellation on policy-agreement regions
and discuss measurement error. Our certificates combine compatible-channel
bounds with finite-sample guarantees under bounded history dependence.

\subsection{Constrained Prediction and Background Cancellation}
Affine clean/noisy risk ordering appears in noise-tolerant ranking
\citep{haddad}. Constrained classification and count-constrained assignments
are established methods \citep{agarwal,matsuo}. Circulation integrality supplies
the batch optimizer \citep{schrijver}. Our construction preserves the incumbent's
prediction counts to cancel an unknown reporting background.
The resulting conditional certificate concerns the selected finite batch.
Its score-regret bound and retained-gain evaluation quantify the consequences
of that constraint.

\subsection{Stateful Agent Memory}
Dependent truth discovery predates agent-memory applications \citep{dong}.
Nous supplies categorical beliefs and observational updates \citep{nous}.
MemTX and MemTxn study transactional memory commits \citep{memtx,memtxn},
and StateMemBench evaluates evolving states using executable event programs
\citep{statemem}. Our contribution concerns statistical task separation and
auditable policy revision, with Nous providing the integration target.

\section{Limitations}
The learning results concern specified HMM families with complementary
emissions, balanced states and positive anchor orientation. Their asymptotic
rates can have large constants. The broader-class lower bounds use different
least-favorable subfamilies, and neither result establishes pointwise superiority
of a memory learner.

Population certification assumes fresh iid calibration records from the
deployment law and a positively oriented shared-background reference channel.
Deviation and copying budgets require external justification.
The available reports do not identify these budgets.
The categorical channel is restricted, and the finite-sample bounds can have
low power. Useful witness regions with sufficient probability may not exist.
The sparse-revision rate is established on a balanced conditional subfamily.

Fixed-batch certification requires conditional channel bounds given preceding
reports. Marginal bounds are insufficient. The receipt concerns the audited
batch and does not certify individual revisions or future histories.
Conservation can discard beneficial changes or yield no nonempty subset.
Different unknown signal scales cannot be pooled without further assumptions.
Version checks and artifact hashes do not establish an audit contract or
adversarial authentication.

The action study uses a static cue, shared routing and controlled reporting.
The dialogue study uses one benchmark, seven slots, local models and published
annotations that can contain errors. Audit errors and copying are simulated.
These studies do not validate natural reporting channels, real source lineage,
open-ended attributes or online batch scheduling.
Supplementary Table~\ref{tab:uci-diagnostics} retains the previously exposed
time-series diagnostics. Those data sets do not validate the calibration assumptions.
The certificates concern accuracy gain, not posterior calibration or security
against unrestricted adversaries. Independent external proof review and
replication remain pending.

\section{Conclusion}
We study decision learning and revision certification when source reliability
is unknown. A shared-observation HMM family separates their sample requirements
from those of source estimation. Compatible-channel bounds certify population
policy gain, while prediction-count conservation enables conditional fixed-batch
certification. The experiments show increased acceptance under weak audits and
the accuracy gain sacrificed by balancing. Nous integrates these certificates
as policy-level evidence for memory revision.

\ifdefined\nousijar
\section{Declaration of Generative AI and AI-Assisted Technologies}
\else
\section{AI Use Statement}
\fi
Generative artificial intelligence tools assisted theory and proof development,
experimental design, synthetic-data generation, implementation and testing.
They also assisted literature analysis, result interpretation and manuscript
preparation. The author is responsible for verifying this work and the final content.

\ifdefined\nousjournal
\section{Data and Code Availability}
The accompanying research archive contains the protocols, input data,
simulator traces, audit statistics, source code and replay checks.
It includes the licensed public MultiWOZ~2.4 data and the evaluated models
and predictions. The versioned archive and file checksums are available at
\href{https://github.com/Pranavsingh431/nous-state/releases/tag/jmlr-submission-2026-10-04}{the Nous reproducibility release}.

\fi

\section{Ethics Statement}
We did not recruit new human participants or disclose private dialogues.
The text experiment reused public benchmark conversations under their
redistribution license. The interface distinguishes population-policy and
fixed-batch certificates to reduce the risk of interpreting aggregate evidence
as an individual safety guarantee.

\ifdefined\nousjmlr
\acks{The author received no funding or third-party support for this work
during the past 36 months and declares no relevant competing financial
interests. An earlier version of this manuscript is available as
\url{https://arxiv.org/abs/2610.00094}. The present revision adds continuous
source uncertainty, a matching powered-certification lower bound,
disagreement-conditioned auditing, certifiable balanced batches and the
MultiWOZ text-policy evaluation. The original Nous system paper is a separate
related preprint \citep{nous}.}
\fi

\fi
\def\nousappendices{1}
\ifdefined\nousappendices
\clearpage
\appendix
\setcounter{table}{0}
\setcounter{figure}{0}
\setcounter{equation}{0}
\setcounter{theorem}{7}
\renewcommand{\thetable}{S\arabic{table}}
\renewcommand{\thefigure}{S\arabic{figure}}
\renewcommand{\theequation}{S\arabic{equation}}
\renewcommand{\theHtable}{S\arabic{table}}
\renewcommand{\theHfigure}{S\arabic{figure}}
\renewcommand{\theHequation}{S\arabic{equation}}
\iftrue

This supplement contains proofs, evaluation protocols and extended results.
We use HMM for hidden Markov model, iid for independent and identically
distributed, TV for total variation, and KL for Kullback--Leibler divergence.
Last-write-wins (LWW) denotes the latest-report baseline.
We use linear programming (LP) for gain endpoints and batch selection.
\fi
\section{Proof of the Shared-Experiment Theorem}\label{app:shared}
Write $A=1/4$, $B=3/8$, $\eta_s(x)=Ax_1+sB\chi(x)$,
$c\in[1,3/2]$, and $\lambda=\ell/c^2$. Since $c(A+B)\leq15/16$,
all emissions are positive. The anchor parameter $a=cA$ and transition
parameter both lie in the original model class. Hidden-path summation gives
\begin{equation*}
 P_{s,c}(x_0,x_1,y)=\frac{1+\ell e_0e_1+\ell Ay e_1+
                      (\ell^2/c^2)Ay e_0}{128},\quad e_t=\eta_s(x_t).
\end{equation*}
Every cell is at least $1/256$. Marginally $X_1$ is uniform, so both
$g_s=s\chi(X_1)$ and $b=X_{1,1}$ predict either label with probability $1/2$.

\subsection{Learning and a Matching Lower Bound}
The orientation moment satisfies
$\E[\chi(X_1)Y]=s\ell AB=3s\ell/32$, independent of $c$.
Its empirical sign, with fixed tie breaking, learns $s$ with failure at most
$\delta$ for
\begin{equation*}
 n_{\rm learn}\geq(2048/9)\ell^{-2}\log(1/\delta).
\end{equation*}
The full-history posterior mean has numerator
$c[\eta_s(X_1)+\lambda\eta_s(X_0)]$ and denominator
\begin{equation*}
1+\lambda c^2\eta_s(X_0)\eta_s(X_1)>0.
\end{equation*}
Its sign is $g_s$, because $B-A-\lambda(A+B)\geq3/64$.
Thus the learner finds the full-history Bayes rule, despite using only a
particular moment of the full available records.

For fixed $c$, the two $s$ laws differ in four orthogonal Walsh characters,
giving $\mathrm{KL}(P_{+,c},P_{-,c})\leq(9/32)\ell^2$.
For $n\leq16/(9\ell^2)$, product KL and Pinsker give TV at most $1/2$.
Decode $s$ from the sign of $\E_U[g(H)\chi(X_1)]$, where $U$ is uniform on
all 64 histories. Clean regret is
\begin{equation*}
 \frac c2\E_U[(s\chi(X_1)-g(H))
                 (\eta_s(X_1)+\lambda\eta_s(X_0))].
\end{equation*}
If decoding is wrong, the parity contribution is at least $B$ and the
remaining contributions have total magnitude at most $A+\lambda(A+B)$.
Hence regret is at least $c[B-A-\lambda(A+B)]/2\geq3/128$.
Le Cam implies worst-case failure at least $1/4$ for smaller tolerances.
Randomized rules use their conditional mean prediction in the decoder.

\subsection{Certification Using the Future Report}
Given $Z_1,H$, the future report has mean $\lambda aZ_1$. It is therefore
already a symmetric shared-background audit of current state, with
$\rho=\ell A/c$ and $u_0=u_1=(1-\rho)/2$. No second sensor or clean-label
sample is supplied. Candidate and incumbent accuracies are $(1+cB)/2$ and
$(1+cA)/2$, so
\begin{equation*}
 \Delta=c/16,\qquad m_0=m_1=0,\qquad\Delta_Y=\rho\Delta=\ell/64.
\end{equation*}
Use the certificate of Theorem~\ref{thm:witness} on a fresh iid split with
both witness regions equal to the whole history space. Conditional validity
also covers the incorrectly learned orientation, whose gain is
$-c(A+B)/2<0$. Conditional on correct learning, put
$t=\sqrt{2\log(4/\eta)/n_{\rm audit}}$. One-sided Hoeffding for the proxy
gain and two-sided Hoeffding for the binary mass change yield
$\hat\Delta_Y\geq\ell/64-t$ and $|\hat m_1|\leq t$ except with probability
at most $3\eta/4$. Since $U_k\leq1$ and $\hat m_0=-\hat m_1$,
\begin{equation*}
 L\geq\ell/64-2t-\sqrt{8\log(2/\alpha)/n_{\rm audit}}.
\end{equation*}
Thus
\begin{equation*}
 n_{\rm audit}>4096\bigl[2\sqrt{2\log(4/\eta)}+
                  \sqrt{8\log(2/\alpha)}\bigr]^2\ell^{-2}
\end{equation*}
suffices for acceptance with probability at least $1-\eta$ conditional on
correct learning. Joint correct learning and acceptance has probability at
least $1-\delta-\eta$, while non-improving acceptance is bounded by $\alpha$.
These sufficient constants are large. The claim is a rate comparison, not
practical small-sample efficiency or superiority of this gate on this family.

For a matching certification lower bound, fix the same proposal
$g_+=\chi(X_1)$ against $b=X_{1,1}$ in the worlds $(s,c)=(+,1)$ and $(-,1)$.
Its clean gains are $1/16$ and $-5/16$, respectively. Let $T$ be any test
(possibly randomized) based on $n$ full records, with
$\E_-T\leq\alpha$ and $\E_+T\geq1-\eta$. Then
\begin{equation*}
 1-\alpha-\eta\leq\TV(P_+^n,P_-^n)
 \leq\sqrt{9n\ell^2/64},\qquad
 n\geq\frac{64}{9}(1-\alpha-\eta)^2\ell^{-2}.
\end{equation*}
The lower bound counts all records supplied to the test, including any
training records it reuses. It is a powered fixed-candidate testing statement,
not a lower bound for a rule allowed always to decline. Combined with the
fresh-split upper bound, it establishes the fixed-confidence certification rate.

\subsection{Matching Source-Estimation Bounds}
For each fixed $s$, switching between $c=1$ and $c=3/2$ preserves all order-$\ell$ terms of the
full law. The difference is
$\ell^2(5/9)Ay\eta_s(x_0)/128$.
Using $\E_U[\eta_s(X_0)^2]=13/64$ and the cell floor gives, in either direction,
\begin{equation*}
 \mathrm{KL}(P_{s,1},P_{s,3/2})\leq
 \chi^2(P_{s,1},P_{s,3/2})\leq(325/41472)\ell^4.
\end{equation*}
For $n\leq20736/(325\ell^4)$, TV is at most $1/2$. The anchor parameters
$1/4$ and $3/8$ differ by $1/8$. Testing at midpoint $5/16$ yields
$\max_c\Prob(|\hat a-a|\geq1/16)\geq1/4$. This bound applies to all
observed records pooled across learning and certification splits, even if $s$
is known. Receipts cannot add information beyond those records.

Conversely $\E[X_{0,1}Y]=\ell^2A^2/c^2$ for every $c\in[1,3/2]$,
independently of $s$. For the empirical moment $\bar T$, set
\begin{equation*}
 \hat u=\operatorname{clip}_{[4/9,1]}\left(\frac{16\bar T}{\ell^2}\right),
 \qquad \hat a=\frac{1}{4\sqrt{\hat u}}.
\end{equation*}
Clipping cannot increase error about $u=c^{-2}$. The derivative of
$u\mapsto(4\sqrt u)^{-1}$ has magnitude at most $27/64$ on this interval.
Thus $|\hat a-a|\leq(27/4)\ell^{-2}|\bar T-\E T|$, and Hoeffding gives
\begin{equation*}
 \sup_{s,c}\Prob(|\hat a-a|\geq\tau)
 \leq2\exp[-8n\tau^2\ell^4/729].
\end{equation*}
This proves the matching upper rate on the continuous-source family for any
fixed $\tau>0$, not on the entire original model class. The two endpoint
worlds suffice for its minimax lower bound. Its upper bound covers the interval.

Exact rational checks enumerate all observable cells and hidden paths at
three persistence values, checking Bayes signs, clean/proxy gains, balanced
prediction masses and both chi-squared bounds. These support the algebra.
The inequalities above establish the interval-wide result. The construction
is deliberately restricted and still minimax, not a pointwise statement.
Additional exact checks cover interior source scales and both fixed-candidate
gain signs. Finite checks are regression tests. The interval-wide proof uses
the inequalities above, not a grid extrapolation.

\section{Proof of the Common-Experiment Separation}\label{app:matched}
This appendix proves the uniform decision bound and the two task-specific lower bounds.
\subsection{Uniform Decision Upper Bound}
Conditional emission independence and the symmetric transition imply
$\E[Y\mid Z_1,H]=\lambda a Z_1$. Hence conditional report error is symmetric
and independent of $H$ given $Z_1$. For the finite class $\mathcal G$ of all
$2^{64}$ deterministic policies, Hoeffding and a union bound give simultaneous
risk-estimation error at most
$\epsilon_n=\sqrt{\log(2|\mathcal G|/\delta)/(2n)}$.
An empirical minimizer has excess future risk at most $2\epsilon_n$.
The positive symmetric bridge preserves the conditional Bayes rule and divides
this excess by $\lambda a\geq\ell/9$, giving \eqref{eq:upper}.

\subsection{Nontrivial Decision Lower Bound}
Set $a=.3$, $b=.5$, $\lambda=\ell$, and let
$\chi(x)=x_1x_2x_3$, $\eta_s(x)=a x_1+s b\chi(x)$ for $s\in\{-1,+1\}$.
Define $f_s(x\mid z)=[1+z\eta_s(x)]/8$.
These are positive complementary emissions with anchor informativeness $a$.
Summing the three latent states gives the full observed law
\begin{equation*}
 P_s(x_0,x_1,y)=\frac{1+\lambda\eta_s(x_0)\eta_s(x_1)
 +\lambda a y\eta_s(x_1)+\lambda^2 a y\eta_s(x_0)}{128}.
\end{equation*}
Every cell exceeds $1/256$, since $|\eta_s|\leq.8$ and $\lambda\leq1/8$.
Writing $\chi_t=\chi(x_t)$, the difference is
\begin{equation*}
 P_+-P_-=\frac{2ab}{128}
 \left[\lambda(x_{0,1}\chi_1+\chi_0x_{1,1}+y\chi_1)
       +\lambda^2y\chi_0\right].
\end{equation*}
The four Walsh characters in brackets are orthogonal under the uniform
128-cell law. Therefore
\begin{align*}
 \KL(P_+\Vert P_-)&\leq\sum_o\frac{(P_+(o)-P_-(o))^2}{P_-(o)}\\
 &\leq8a^2b^2(3\lambda^2+\lambda^4)
 \leq32a^2b^2\lambda^2.
\end{align*}
Pinsker and product additivity give
$\TV(P_+^n,P_-^n)\leq4ab\lambda\sqrt n\leq1/2$ for
$n\leq1/(64a^2b^2\ell^2)$.

Summing out $Y$, the unnormalized current posterior margin is
$\eta_s(x_1)+\lambda\eta_s(x_0)$. Since
$b-a-\lambda(a+b)\geq.1$, its sign is $s\chi_1$. Both Bayes accuracies are
$(1+b)/2=.75$. If $\E_U$ denotes uniform averaging over the 64 histories,
any realized deterministic policy has excess
\begin{equation*}
 R_s(g)-R_s(g_s^*)=\frac b2-\frac12
 \E_U[g(\eta_s(x_1)+\lambda\eta_s(x_0))].
\end{equation*}
Decode the model sign using $\sgn(\E_U[g\chi_1])$, with a fixed tie rule.
When this decoding is wrong, the main parity term cannot be positive. The
remaining terms have magnitude at most $a+\lambda(a+b)$. The excess is thus
at least $[b-a-\lambda(a+b)]/2\geq.05$. Le Cam's two-point inequality gives
maximum decoding error at least $(1-\TV)/2\geq1/4$, and therefore the same
lower bound for probability of excess at least $.05$.
Randomized learners are covered by adjoining their independent random seed.
The same bounded-term argument applies to a randomized policy's expected
prediction. This proves the quadratic decision lower bound.

\subsection{Parameter Lower Bound without Discarding Observations}
Let $U_t=Z_tE_t$, where the independent errors have $\E[E_t]=a$. Set
$X_{t,1}=U_t$ and $X_{t,j}=U_tB_{t,j}$ for $j=2,3$, with independent $B$'s
flipping with known probability $.1$. These emissions have full support and
belong to the same class. Conditional on the observed anchor triple
$(U_0,U_1,U_2)$, the four additional source entries are generated by a
parameter-independent channel. Conversely the full observation contains the
anchor triple. The KL chain rule therefore gives exact equality between
full-record KL and anchor-triple KL.

The latter law is
\begin{equation*}
 P_{a,\lambda}(u,v,w)=
 \frac{1+\lambda a^2(uv+vw)+\lambda^2a^2uw}{8}.
\end{equation*}
Choose $(a_0,\lambda_0)=(.5,\ell)$ and
$(a_1,\lambda_1)=(.75,4\ell/9)$. Both are in $\mathcal M_\ell$.
Their lag-one coefficients coincide at $\ell/4$, while lag-two coefficients
differ by $d=5\ell^2/36$. Every scalar cell is at least $1/16$, giving
\begin{equation*}
 \KL(P_0\Vert P_1)\leq2d^2=25\ell^4/648,\qquad
 \TV(P_0^n,P_1^n)\leq5\ell^2\sqrt n/36.
\end{equation*}
For $n\leq324/(25\ell^4)$, total variation is at most $1/2$.
Thresholding an estimate at $5/8$ turns it into a test. An incorrect test
implies estimation error at least $1/8$, so
\begin{equation*}
 \inf_{\hat a}\max_{j\in\{0,1\}}
 \Prob_j\{|\hat a-a_j|\geq1/8\}\geq1/4.
\end{equation*}
The two minimax lower bounds use different least-favorable subfamilies inside
one model class. The parameter result is not obtained by withholding source
observations available to the decision learner.

\begin{figure}[t]
\centering
\includegraphics[width=\linewidth]{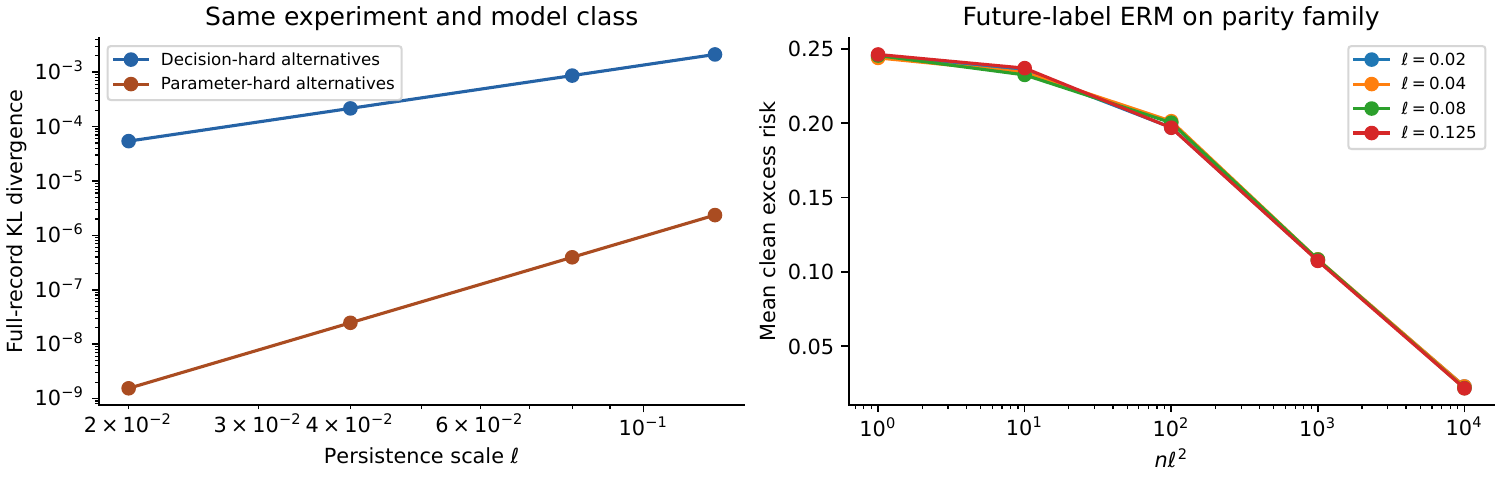}
\caption{Full-record divergences and future-label risk in the broader model
class. KL denotes Kullback--Leibler divergence and ERM denotes empirical risk minimization.}
\label{fig:matched}
\end{figure}

Figure~\ref{fig:matched} compares the exact 128-cell laws and future-label
empirical risk minimization over 500 repetitions per point.
Decision-hard alternatives have quadratic KL scaling, while source-hard
alternatives have quartic scaling. The right panel shows risk against the
scaled sample count. The simulation uses the broader class of
Theorem~\ref{thm:matched}.

\section{Details of Compatible-Channel Certification}\label{app:certificate}
We prove the identified gain interval and its finite-sample witness certificate.
\subsection{Proof of Theorem~\ref{thm:sharp}}
Nonnegativity of $p$ requires the box constraints. Conversely every point
in the box gives a positive $\rho$ and normalized nonnegative $p(h)$, and hence
a joint law $\Prob(H=h)p_Z(h)C_{Y,Z}$ with the observed distribution. Equation
\eqref{eq:identity} yields $F$. The box is compact and connected and its
denominator stays positive, giving attainable endpoints and all intermediate
values. Minimizing the numerator independently sets positive-$m$ coordinates
to $\ell_k^Y$ and negative-$m$ coordinates to zero, proving \eqref{eq:sharp}.
Linear-fractional optimization obtains the endpoints (Appendix~\ref{app:certificate}).
\hfill$\blacksquare$

\subsection{Proof of Theorem~\ref{thm:witness}}
The simultaneous event $u_k\leq U_k$ has probability at least $1-\alpha/2$.
On it, the empirical worst numerator over $[0,U]$ is no larger than the
empirical numerator at the fixed true $u$. Its iid summands
\begin{equation*}
 T_i(u)=\ind\{g_i=Y_i\}-\ind\{b_i=Y_i\}-u_{g_i}+u_{b_i}
\end{equation*}
lie in $[-2,2]$ and have mean $\rho\Delta$. One-sided Hoeffding with error
$\alpha/2$ gives the stated radius. A union bound completes the proof without
conditioning on data-derived $U$, or taking a union over fitted vertices.
\hfill$\blacksquare$

\subsection{Linear Program for Sharp Endpoints}
Set $t=1/(1-\sum_k u_k)$ and $v=tu$. Minimize or maximize
$\Delta_Yt-m^{\mathsf T}v$ subject to
\begin{equation*}
 t-\sum_k v_k=1,\qquad 0\leq v_k\leq\ell_k^Y t,\qquad t\geq1.
\end{equation*}
The inverse is $u=v/t$. The population implementation refuses constant $q$:
then $\sum_k\ell_k^Y=1$ and the compact positive-denominator argument no
longer applies. This restriction is unnecessary for the finite-sample witness
test, which remains valid for constant channels with $\rho>0$.

\subsection{Binary Sign Condition}
Use labels $0,1$, write $r(h)=\Prob(Y=1\mid H=h)$ and
$m=\Prob(g=1)-\Prob(b=1)$. Since the two mass changes are $(-m,m)$,
the sharp positive-gain condition is
\begin{equation*}
 \Delta_Y>
 \begin{cases}
 m\min_h r(h),&m\geq0,\\
 (-m)(1-\max_h r(h)),&m<0.
 \end{cases}
\end{equation*}
For example, let $H$ be a fair bit and the clean state equal $H$ with
probability $.8$. Let $g=H$, $b=0$, $\rho=.8$ and background probabilities
$(.25,.75)$. Then $r(0)=.31$, $r(1)=.79$, $m=.5$ and $\Delta_Y=.29$.
The sharp numerator is $.135>0$, so a one-direction correction is identifiable.
The unrestricted rule $\Delta_Y>|m|$ cannot accept it.

\subsection{Exact Witness Bound}
If class $k$ has $N_k$ records in its fixed region and $S_k$ of their reports
equal $k$, set
\begin{equation*}
 U_k=\begin{cases}
 1,&N_k=0\text{ or }S_k=N_k,\\
 \operatorname{BetaQuantile}_{1-\alpha/(2K)}(S_k+1,N_k-S_k),&\text{otherwise}.
 \end{cases}
\end{equation*}
Conditionally on the region count, $S_k$ is binomial with the regional
report probability. Exact conditional coverage therefore implies unconditional
coverage, including random region counts. Regions may overlap: the argument
uses a union bound, not independence. Learning regions from the same reports
used in this calculation is not covered by the theorem.

\subsection{Finite-History Consistency}
For each class, select a minimum empirical report-probability history from an
independent training sample, with fixed tie breaking. If every history has
positive probability and both training and calibration sizes diverge, uniform
convergence on the finite alphabet ensures selected population frequencies
converge to $\ell_k^Y$. Counts in the selected calibration regions diverge and
their binomial upper bounds converge to the same values. Empirical policy
gains/masses converge and the Hoeffding radius vanishes. Thus every fixed pair
with strictly positive sharp numerator is eventually accepted with probability
tending to one. No corresponding claim is made at a zero numerator.

\subsection{Unrestricted-Channel Comparator}
For binary $d_i=g_i-b_i$, the earlier test requires both means of
$d_iY_i$ and $-d_i(1-Y_i)$ to be positive. Each is a $\{-1,0,1\}$ score.
With positive/negative counts $W,L$, its exact one-sided sign-test p-value is
$\Prob\{\operatorname{Binomial}(W+L,1/2)\geq W\}$, or one if $W+L=0$.
Both tests must reject at level $\alpha$. Intersection-union logic gives a
level-$\alpha$ guarantee. An entirely upward revision makes the second score
nonpositive everywhere. For $K>2$, the earlier implementation requires every
mean of $D_i-\ind\{g_i=k\}+\ind\{b_i=k\}$ to exceed
$\sqrt{8\log(1/\alpha)/n}$. Its population condition is sufficient, not sharp.

\subsection{Verification}
Tests enumerate all vertices of random compatible boxes for $K=2,3,5$,
compare their gain extrema with the LP, reconstruct the attaining clean laws,
and verify that they reproduce the observed law. Fixed-seed binary/categorical
one-direction examples pass the new certificate. Harmful and empty-sample
examples fail. These checks supplement but do not replace the analytic proof.

\section{Mutable-State Protocol and Complete Results}\label{app:ledger}
At each step an executable state record attempts a change with probability $.12$.
The generator includes failed-version and duplicate-command controls.
Four report snapshots have known synthetic event/arrival times and roots.
Policies receive reports, not initial state, commands or withheld ledger truth.
The candidate selection set comprises native/root-capped Nous, smoothed
suffix-one/two future predictors and independent/joint-emission HMMs. Selection
uses only development next-anchor accuracy. The selected policy is an HMM or
a suffix-two predictor. Table~\ref{tab:replay-summary} records selection and
replay quantities.

\begin{table}[ht]
\centering\small
\begin{tabular}{lrrr}
\toprule
Method & Independent & Copy & Delayed copy \\
\midrule
Arrival LWW & 74.82 $\pm$ 0.39 & 68.41 $\pm$ 0.91 & 58.34 $\pm$ 0.50 \\
Event-time latest & 74.82 $\pm$ 0.39 & 68.41 $\pm$ 0.91 & 67.91 $\pm$ 0.43 \\
Original Nous updater & 77.86 $\pm$ 0.14 & 67.72 $\pm$ 0.94 & 55.81 $\pm$ 0.86 \\
Root-deduplicated Nous & 77.86 $\pm$ 0.14 & 74.51 $\pm$ 0.71 & 73.78 $\pm$ 0.81 \\
Current anchor & 85.31 $\pm$ 0.26 & 80.10 $\pm$ 0.92 & 80.03 $\pm$ 0.81 \\
Current majority & 89.52 $\pm$ 0.23 & 67.92 $\pm$ 0.79 & 67.56 $\pm$ 0.45 \\
Future predictor (suffix 2) & 90.50 $\pm$ 0.51 & 83.06 $\pm$ 0.60 & 82.81 $\pm$ 0.70 \\
Independent-emission HMM & 93.97 $\pm$ 0.19 & 67.49 $\pm$ 0.67 & 67.17 $\pm$ 0.47 \\
Joint-emission HMM & 93.70 $\pm$ 0.27 & 85.40 $\pm$ 0.62 & 85.28 $\pm$ 0.56 \\
Certified Nous policy & 93.97 $\pm$ 0.19 & 84.59 $\pm$ 0.77 & 84.38 $\pm$ 1.87 \\
\bottomrule
\end{tabular}

\caption{Mutable-state accuracy (\%). Values are means with sample standard
deviations. LWW denotes last-write-wins and HMM denotes hidden Markov model.}
\label{tab:ledger-accuracy}
\end{table}

\begin{table}[ht]
\centering\small
\begin{tabular}{llrrr}
\toprule
Scenario & Run & Gain (pp) & 95\% CI (pp) & Harmful / beneficial \\
\midrule
independent & 1 & 19.76 & [18.52, 21.00] & 102 / 1090 \\
independent & 2 & 18.86 & [17.59, 20.13] & 144 / 1087 \\
independent & 3 & 18.82 & [17.56, 20.08] & 136 / 1077 \\
copy & 1 & 15.66 & [14.29, 17.03] & 277 / 1060 \\
copy & 2 & 17.52 & [16.15, 18.89] & 249 / 1125 \\
copy & 3 & 15.38 & [13.94, 16.82] & 350 / 1119 \\
delayed copy & 1 & 28.12 & [26.53, 29.71] & 320 / 1726 \\
delayed copy & 2 & 23.52 & [21.82, 25.22] & 492 / 1668 \\
delayed copy & 3 & 26.48 & [24.88, 28.08] & 349 / 1673 \\
\bottomrule
\end{tabular}

\caption{Selected policy gain and individual revision outcomes.
LWW denotes last-write-wins, pp denotes percentage points, and CI denotes confidence interval.}
\label{tab:ledger-gains}
\end{table}
Tables~\ref{tab:ledger-accuracy} and \ref{tab:ledger-gains} report accuracy and
paired revision outcomes. Table~\ref{tab:replay-summary} records the local runtime.
Full-array audits reconstruct fixed policies, predictions and certificates.
They replay every state record and Dimension receipt without changing selection.
Where reconstruction reruns HMM fitting, it uses the original training split
and fixed seed, never held-out truth.

\section{External Simulator Protocol and Paired Results}\label{app:minigrid}
MiniGrid 3.0.0 and Gymnasium 1.3.0 execute the S7, S13 and S17Random Memory
environments. Controller smoke tests use seeds 0--11. Final episode
seeds are independently derived from the protocol namespace, environment,
run, split and record index and checked for collisions after conversion to
32 bits. Fixed policies precede all final calibration and test generation.

\subsection{Input and Action Audit}
The controller first returns west to the cue room, observes the object while
facing north, traverses east to the junction and observes the upper branch
object. It uses canonical partial observations for routing, never the hidden
success coordinate. The noisy wrapper receives the cue to generate reports,
then every memory policy receives only the reports and recorded copying roots.
Branch choice uses its predicted cue and the visible branch-object identity.
The clean cue retained by the runner is used for evaluator checks only.
No history of pristine cue observations is supplied to a memory baseline.
All methods share routing, including its known layout conventions. This is
memory-module evaluation, not learned navigation.

Calibration resets the same environment family but deliberately stages the
camera at the cue. It does not claim a calibration action rollout. This
changes routing cost, not the cue/report population used by the fixed policy
and witness. Test has no camera-staging shortcut. The runner deep-copies the
junction state and executes both possible terminal decisions to obtain exact
paired action outcomes. The audit independently replays both saved action
sequences from reset for every saved episode seed.
It reconstructs the certificates and checks calibration/test seeds and hashes.
Table~\ref{tab:replay-summary} records these quantities, and
Table~\ref{tab:minigrid-gains} reports the paired gains.

\begin{table}[ht]
\centering\small
\begin{tabular}{llrrr}
\toprule
Environment/run & Incumbent & Accepted & Gain (pp) & Paired 95\% CI \\
\midrule
S13/1 & LWW & 0 & 0.00 & [0.00, 0.00] \\
S13/1 & Constant & 1 & 23.70 & [20.46, 26.94] \\
S13/2 & LWW & 0 & 0.00 & [0.00, 0.00] \\
S13/2 & Constant & 1 & 24.10 & [20.90, 27.30] \\
S13/3 & LWW & 1 & 13.10 & [10.50, 15.70] \\
S13/3 & Constant & 1 & 24.90 & [21.60, 28.20] \\
S17Random/1 & LWW & 1 & 13.50 & [10.81, 16.19] \\
S17Random/1 & Constant & 1 & 24.40 & [21.15, 27.65] \\
S17Random/2 & LWW & 0 & 0.00 & [0.00, 0.00] \\
S17Random/2 & Constant & 1 & 22.60 & [19.46, 25.74] \\
S17Random/3 & LWW & 0 & 0.00 & [0.00, 0.00] \\
S17Random/3 & Constant & 1 & 24.80 & [21.56, 28.04] \\
S7/1 & LWW & 1 & 15.20 & [12.38, 18.02] \\
S7/1 & Constant & 1 & 23.30 & [19.93, 26.67] \\
S7/2 & LWW & 1 & 14.60 & [11.73, 17.47] \\
S7/2 & Constant & 1 & 26.20 & [23.01, 29.39] \\
S7/3 & LWW & 0 & 0.00 & [0.00, 0.00] \\
S7/3 & Constant & 1 & 24.50 & [21.30, 27.70] \\
\bottomrule
\end{tabular}

\caption{MiniGrid deployed gains and paired confidence intervals.
LWW denotes last-write-wins, pp denotes percentage points, and CI denotes confidence interval.}
\label{tab:minigrid-gains}
\end{table}

Saved result files additionally contain harmful
and beneficial overrides, override fraction and actual environment reward.
The certificate targets success/accuracy, not MiniGrid's step-penalized reward.
Changing the final turn sequence can change reward even at equal success.
The source-0 baseline is not oracle-tuned: its report index and evaluation
were fixed in the protocol, although the controlled generator gives it higher
accuracy than any other individual source.

\section{Categorical Control Protocols}\label{app:controls}
For categorical controls $H$ is uniform on $K$ classes. Except for null cases,
$\Prob(Z=H\mid H)=.8$ and each alternative state has probability $.2/(K-1)$.
The one-direction candidate is $g(H)=H$. Its incumbent maps only $H=K-1$
to zero. Thus its true gain is $(.8-.2/(K-1))/K$. The harmful candidate is
constant zero against $b(H)=H$, with gain $1/K-.8$. In null cases $Z$ is
uniform independent of $H$, $g(H)=H$ and $b(H)=(H+1)\bmod K$.
Reports preserve $Z$ with probability $\rho$ and otherwise draw from
$(v_0,(1-v_0)/(K-1),\ldots)$. Fixed witnesses are $A_k=\{H\ne k\}$.

The lag controls use $\rho=.8(.8)^d$. This is a controlled effective
lag-to-signal law, not an empirical estimate from natural trajectories. It
matches the decay of the identity component under a symmetric categorical
transition channel, with the resulting background encoded in the report law.
Longer lags therefore weaken the report evidence explicitly.

Table~\ref{tab:replay-summary} records the independent calibration repetitions
for the categorical, phase and sequential-release controls.
Fixed-cell repetitions retain seeds, decisions, empirical margins and radii.
Sequential runs retain decisions with seeds recoverable from the protocol.
Generator hashes allow regeneration. Exact Clopper--Pearson
intervals for acceptance frequencies are stored for every cell. In a 100-run
cell, 0/100 gives a 95\% upper bound of about 3.62\%, and 100/100 a lower
bound of about 96.38\%. These do not replace theoretical type-I control.

The sharp interval is computed from the known population law only to diagnose
identification in the synthetic study. Neither that law, true $\rho$, true
$u$, nor clean states are supplied to the finite-sample certificate. The
phase follow-up was motivated by the initial zero-power result and is labelled
accordingly. All 45 cells, including its 20 zero-power cells, are reported.

\subsection{History-Dependent Channel Counterexample}
Let $H$ be a fair bit, $g=H$, $b=1-H$. Two worlds have
$\Prob(Z=H\mid H)=.8$ and $.2$, respectively. At a history with
$\Prob(Z=1\mid H)=.8$, set $(q_0,q_1)=(.1,.6)$. If it equals $.2$, set
$(q_0,q_1)=(.4,.9)$. Both worlds have $\Prob(Y=1\mid H)=.5$ and positive
channel orientation, but their gains are $+.6$ and $-.6$. Arbitrary
history-dependent reporting therefore defeats observable-only comparison.
Likewise, independently replacing a report by the candidate's prediction with
probability $\beta$ gives
$\Delta_Y^{\mathrm{copy}}=(1-\beta)\Delta_Y^{\mathrm{base}}
+\beta\Prob(g\ne b)$, potentially rewarding a harmful revision.
The original paired signal is attenuated as well as contaminated.
Known source roots help audit dependence but do
not by themselves validate the assumed channel.

\section{Additional Arithmetic and Assumption Scrutiny}
Symbolic summation over hidden paths verifies the matched construction's
observable law. Exact integer Walsh inner products and rational endpoint
checks verify its orthogonality and positive posterior margin. These checks
support particular proof steps, not automated verification of the full theorem.
For a binary one-direction comparison, complete conditional-binomial summation
evaluates the witness certificate across specified laws using sample sizes
128, 512 and 2048, three gains, and three report channels.
Table~\ref{tab:arithmetic-check} reports the largest computed non-improving
acceptance probability at $\alpha=.05$. The calculation covers the specified
laws, not the supremum over all null laws. Probabilities use
floating-point arithmetic, not exact rational arithmetic or Monte Carlo.
A separate exhaustive small-sample enumeration checks the summation against
the implemented certificate.

An intentionally invalid control chooses witness masks using calibration
labels instead of histories: it accepts a zero-population-gain proposal,
whereas fixed history-only witnesses reject on the same records. This violates
the witness certificate's pre-calibration witness requirement, not its conclusion. The low-level
witness-mask interface cannot verify that requirement from boolean arrays alone. Callers
must enforce it. Neither this audit nor a saved mask hash proves independence.
The reproducibility package retains the checks and the counterexample.

These are self-audits. Bibliographic comparison cannot establish that no
equivalent result exists elsewhere. Automated verification does not imply
independent scientific scrutiny or author approval.

\section{Bounded Channel Misspecification}\label{app:tv}
Let $C_h(\cdot\mid z)$ be the actual audit channel and let a single reference
channel be $C_0(\cdot\mid z)=\rho e_z+u$, with $\rho>0$, $u\geq0$ and
$\sum_k u_k=1-\rho$. Assume
\begin{equation*}
 \mathop{\mathrm{ess\,sup}}_{h,z}
 \mathrm{TV}\bigl(C_h(\cdot\mid z),C_0(\cdot\mid z)\bigr)\leq\epsilon,
 \qquad \mathrm{TV}(P,Q)=\tfrac12\sum_k|P_k-Q_k|.
\end{equation*}
This permits bounded history dependence. The clean deployment law and iid
calibration assumptions remain unchanged. Let $U_k$ be the existing binomial
witness bounds, $V_k=\min(1,U_k+\epsilon)$, and
$\hat d=n^{-1}\sum_i\ind\{g_i\ne b_i\}$.
At $\epsilon=0$ Theorem~\ref{thm:tv} recovers the original witness certificate.
\subsection{Proof of Theorem~\ref{thm:tv}}
Convexity gives $\mathrm{TV}(q(h),\rho p(h)+u)\leq\epsilon$, so
$u_k\leq q_k(h)+\epsilon$. Binomial coverage therefore places the fixed true
$u$ in $[0,V]$ with failure probability at most $\alpha/2$.
The contrast $f_h(y)=\ind\{g(h)=y\}-\ind\{b(h)=y\}$ has oscillation
$2\ind\{g(h)\ne b(h)\}$. The TV expectation inequality bounds its
conditional bias above by $2\epsilon\ind\{g(h)\ne b(h)\}$.
Consequently the fixed-$u$ scores
\begin{equation*}
 S_i=f_{H_i}(Y_i)-u_{g_i}+u_{b_i}
                  -2\epsilon\ind\{g_i\ne b_i\}
\end{equation*}
have mean at most $\rho\Delta$ and lie in $[-2-2\epsilon,2]$.
One-sided Hoeffding at tail $\alpha/2$ gives the displayed radius.
Minimization over $[0,V]$ cannot exceed the empirical true-$u$ score.
The union bound proves the result without conditioning on estimated witnesses.
\hfill$\blacksquare$

\subsection{Conditional Copying and Tight Local Constants}
Suppose a base channel is within $\gamma$ of $C_0$ and is replaced by an
arbitrary report channel with conditional probability $t(h,z)\leq\beta$.
Convexity bounds the resulting TV by
$(1-t)\gamma+t\leq\gamma+\beta-\gamma\beta$. This includes candidate
copying but is stronger than a bound on the average copied fraction.
Moving $\epsilon$ probability from the incumbent label to the candidate label
attains paired bias $2\epsilon$ on disagreement. Removing $\epsilon$ mass
from a background coordinate attains the witness inflation constant one.
These local constants are optimal. Neither global sharpness of the combined
bound nor optimality of its concentration radius is claimed.

Regression tests check zero-budget equivalence, categorical population
envelopes, attained local constants, monotonicity in the budget, and an
explicit null-gain history-dependent attack accepted by the original rule
but rejected at the correct budget. A beneficial control also accepts.
These are implementation checks, not a replacement for a fixed power study.
Budgets are not inferred from noisy reports. Selecting a convenient budget
after inspection does not establish its validity. Fresh-release error
spending continues to apply conditionally when each budget contract holds.

\section{Learned-Witness Power and End-to-End Composition}\label{app:power}
Histories do not necessarily have finite support. For fixed regions with probabilities
at least $\pi>0$, write $r_k=\Prob(Y=k\mid H\in A_k)$ and define
\begin{equation*}
 M=\Delta_Y-\sum_k(m_k)_+\min(1,r_k+\epsilon)-2\epsilon d,
 \qquad d=\Prob(g\ne b).
\end{equation*}
For a desired power failure probability $\eta$, set
\begin{align*}
 a&=\sqrt{2\log(10(K+1)/\eta)/n},\qquad
 b&=\sqrt{\log(10/\eta)/(2n)},\\
 c&=\frac{\sqrt{\log(5K/\eta)}+\sqrt{\log(2K/\alpha)}}{\sqrt{n\pi}},\\
 B&=(4+2\epsilon)\sqrt{\frac{\log(2/\alpha)}{2n}}+(K+1)a+2\epsilon b+c.
\end{align*}
\begin{proposition}[Fixed-region acceptance power]\label{prop:power}
If $n\pi\geq8\log(5K/\eta)$ and $M>B$, the bounded-misspecification
certificate accepts with probability at least $1-\eta$.
\end{proposition}
\subsection{Proof}
Chernoff and a union bound give $N_k\geq n\pi/2$ for all regions except
with probability $\eta/5$. Simultaneous Hoeffding bounds control
$\hat\Delta_Y$ and all $\hat m_k$ within $a$, with failure $\eta/5$,
and $\hat d$ within $b$, with failure $\eta/5$.
Conditional binomial Hoeffding controls each regional frequency above by
$r_k+\sqrt{\log(5K/\eta)/(2N_k)}$, with joint failure $\eta/5$.
Inverting the lower-tail binomial Hoeffding bound gives deterministically
\begin{equation*}
 U_k\leq\min\left(1,\hat r_k+
                 \sqrt{\log(2K/\alpha)/(2N_k)}\right).
\end{equation*}
Thus $\hat V_k\leq\min(1,r_k+\epsilon)+c$ on these events.
Clipping and positive parts are 1-Lipschitz, and $\sum_k(m_k)_+\leq1$.
The empirical numerator is at least $M-(K+1)a-2\epsilon b-c$.
Subtracting the radius and union-bounding the four events proves the claim.
\hfill$\blacksquare$

\subsection{Learned Score Regions}
Use an independent design split to learn arbitrary scores $s_k(h)$ and fix
$J$ threshold regions per class. Conditional on design, suppose every region
has mass at least $\pi$ and the best per-class region has report probability
at most $\ell_k^Y+\xi$, where $\ell_k^Y=\mathop{\mathrm{ess\,inf}}_h q_k(h)$.
Select the smallest empirical report frequency on a separate selection split
of size $N$, breaking ties deterministically. If
$N\pi\geq8\log(2KJ/\eta_{\rm tr})$, all selected regions satisfy
\begin{equation*}
 r_k\leq\ell_k^Y+\xi+
      2\sqrt{\frac{\log(4KJ/\eta_{\rm tr})}{N\pi}}
\end{equation*}
with probability at least $1-\eta_{\rm tr}$. To prove this, union-bound
Chernoff over $KJ$ counts with error $\eta_{\rm tr}/2$, and two-sided
conditional Hoeffding over their frequencies with error $\eta_{\rm tr}/2$.
Empirical minimization costs twice the resulting uniform deviation.
If $S_\epsilon$ denotes $M$ with $r_k$ replaced by $\ell_k^Y$, then
\begin{equation*}
 S_\epsilon>\xi+2\sqrt{\log(4KJ/\eta_{\rm tr})/(N\pi)}+B
\end{equation*}
ensures acceptance probability at least $1-\eta_{\rm tr}-\eta$.
This does not assert that useful large regions exist for every representation.
All population mass and approximation assumptions must be justified. They
are not certified by substituting calibration estimates into this condition.

\subsection{Nous Composition}
Learn a candidate and region family on design data. Select regions on separate
selection data. Lock the policies, regions and externally justified channel
budget. Audit on fresh iid records. We issue a policy-level Nous receipt only
if $L_\epsilon>0$. Conditional on design and selection, Theorem~\ref{thm:tv}
bounds non-improving acceptance by $\alpha$. Integrating over those splits
preserves this guarantee without identifying $\rho$, $u$ or clean posteriors.
The displayed power condition additionally guarantees useful acceptance.
Theorem~\ref{thm:shared} supplies a concrete same-report instance of this composition.
More generally, the HMM learning theorem is one possible candidate learner, not a requirement
that its future-state reports be the pinned audit channel. Repeated fresh
releases retain the stated error-spending guarantee. The receipt concerns
the locked policy population, not correctness of an individual memory value.

\section{Sensitivity and Continuous-Witness Protocol}\label{app:witness-study}
\begin{figure}[ht]
\centering
\includegraphics[width=.95\linewidth]{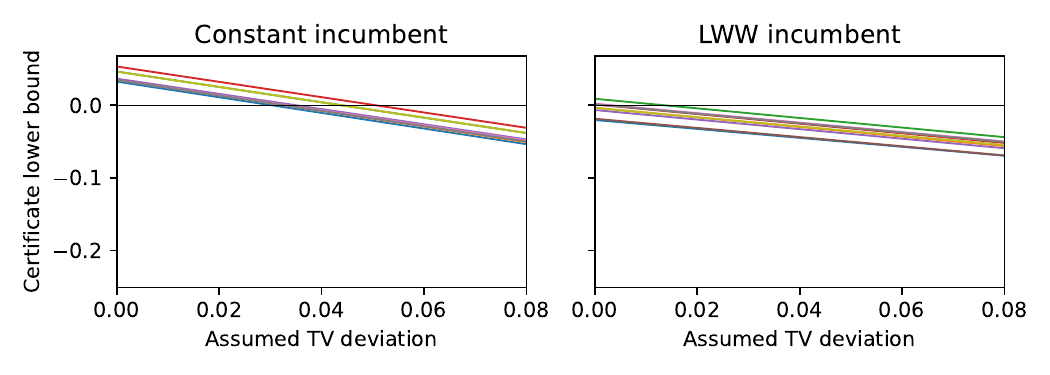}
\caption{MiniGrid lower bounds under increasing assumed channel deviation.
TV denotes total variation and LWW denotes last-write-wins.}
\label{fig:sensitivity}
\end{figure}
Figure~\ref{fig:sensitivity} contains all 18 calibration comparisons.
The LWW panel includes five comparisons that are rejected at zero budget.
The calculation uses calibration data without held-out labels.
The scalar lower bound is continuous, piecewise linear and strictly decreasing
for nonempty samples: clipped witness penalties are nondecreasing, and the
concentration radius strictly increases. Roots are computed numerically to
$10^{-14}$ tolerance. Empty accepting sets are recorded as null, not zero
robustness. The complete per-comparison values are retained in the package.
Selecting an attractive root cannot establish an unknown channel contract.
Actual deviation must still be externally bounded.

\subsection{Learned-Witness Design}
The protocol and runner were written before the first run. There are four
replicates at each $\gamma\in\{0,.03,.10\}$, with 4,000 design, 4,000
selection, 8,000 calibration and 5,000 test records per cell. All 48 split
seeds are distinct and derived from a new namespace. The study is synthetic:
$H\sim N(0,I_{16})$, $\Prob(Z=1\mid H)=\mathrm{sigmoid}(2H_0+H_1)$, and
\begin{equation*}
 \Prob(Y=1\mid H,Z)=.1+.8Z+\gamma\tanh(H_2).
\end{equation*}
Thus the reference has $\rho=.8$, $u=(.1,.1)$, and the per-history/class
TV deviation is at most $\gamma$ by construction. This is not a fitted bound.

On design data fit ridge least squares to $Y$ from an intercept and all 16
coordinates, using unnormalized penalty one including the intercept. Freeze
the candidate $g(h)=\ind\{s(h)>.5\}$ against constant zero. No clean labels
train it. The witness scores are $s_1=s$ and $s_0=1-s$. For each score the
five thresholds are design quantiles $.1,.2,.3,.4,.5$. Independent selection
data choose the threshold with minimum empirical class-specific audit
frequency, with deterministic tie breaking. The coefficients, thresholds,
selected indices and budget are locked before inspecting calibration or test
labels. Only then apply the certificate to fresh calibration.

Comparators use the entire history space or fixed regions $H_0>0$ for class
zero and $H_0\leq0$ for class one. All three use the same candidate, incumbent,
calibration records and separately fixed $\alpha=.05$ tests. Saved results
include nominal and correct-budget rules, region counts, all certificate
terms, test accuracy and paired gain/standard error for every cell. The
correct-budget acceptances are in the main table. No method is selected by
test accuracy or calibration acceptance. Conditional validity holds regardless
of region quality. The stronger population mass/approximation assumptions
of the power proposition are not claimed to have been empirically certified.

The audit regenerates every split, refits design coefficients, reconstructs
selection, verifies locks and all certificates, and recomputes held-out metrics
from saved data. This is self-reproduction, not independent replication.
There is no language model, textual extraction or new external environment
in this experiment. The existing external MiniGrid and mutable-state results
remain unchanged, and their strong baselines and negative outcomes remain.

\section{Sufficient Conditions for a Three-Task Information Separation}\label{app:information-order}
This section extracts a reusable sufficient-condition argument from the HMM.
It uses classical two-point testing and bounded-score concentration. Geometric
connections between functionals and statistical rates precede this work
\citep{donoholiu}, as do first-/second-order identifiability analyses for mixture
parameters \citep{honguyen}. The purpose in this analysis is to state what must additionally
hold for certified improvement, rather than to claim a new rate calculus.

\subsection{Setting and Assumptions}
Let $P_{s,c,\ell}$ be laws of the full observed record $O$ on an arbitrary
measurable space, with $s\in\{-1,+1\}$, nuisance $c\in\mathcal C$, and known
$0<\ell\leq\ell_0\leq1$. The set $\mathcal C$ may be infinite. A corresponding
latent law defines clean risk $R_{s,c,\ell}$, source target $a(s,c,\ell)$, a
known incumbent $b$, and two known policies $g_+,g_-$. All tasks receive iid
records from the same $P_{s,c,\ell}$, never the latent variables. Fix
$0<r<q$, risk tolerance $e_0>0$ and source tolerance $\tau>0$.
All constants below are uniform in $\ell$ and do not conceal its powers.
\begin{enumerate}
\item Decision signal and meaning. $g_s$ is Bayes optimal.
A known measurable $F_\ell(O)\in[-1,1]$ satisfies
$s\E F_\ell\geq a_D\ell^r$, where $a_D>0$.
There is a known decoder $D_\ell$ from output policies to $\{-1,+1\}$ such that
$R(g)-R(g_s)<e_0$ implies $D_\ell(g)=s$. For the lower bound it suffices
that this implication hold on the decision-hard pair in item 3.
\item Sign-valid audit score and power. For each fixed $t\in\{-1,+1\}$,
a known measurable score $V_{t,\ell}(O)\in[-1,1]$ satisfies
\begin{equation*}
 \Delta_t\leq0\ \Longrightarrow\ \E V_{t,\ell}\leq0,
 \qquad \E V_{s,\ell}\geq a_C\ell^r,
 \quad \Delta_t=R(b)-R(g_t),\quad a_C>0.
\end{equation*}
These are assumptions about the clean/observed relationship, not conclusions
from having a nonzero observed correlation. The scores may depend on known
class specifications but not unknown $s,c$ or latent labels.
\item Decision/certification-hard pair. For some $c_D\in\mathcal C$,
the worlds $(+,c_D)$ and $(-,c_D)$ satisfy
\begin{equation*}
 \KL(P_{+,c_D,\ell},P_{-,c_D,\ell})\leq C_D\ell^{2r}.
\end{equation*}
The fixed proposal $g_+$ has positive gain in the first world and
nonpositive gain in the second. The decoder in item 1 applies to this pair.
\item Source-hard pair. For a fixed orientation $s_0$ there are
$c_0,c_1$ with source targets separated by at least $2\tau$, and
\begin{equation*}
 \KL(P_{s_0,c_0,\ell},P_{s_0,c_1,\ell})\leq C_S\ell^{2q}.
\end{equation*}
\end{enumerate}

\begin{proposition}[Task-specific signal and testing conditions]\label{prop:information-order}
Under items 1--4, at fixed error probabilities below $1/4$, learning a policy
with regret below $e_0$ requires $\Theta(\ell^{-2r})$ records.
The learned policy can be certified with type-I error at most $\alpha$ using
a fresh $O(\ell^{-2r})$ audit split, with joint correct learning and acceptance
probability at least $1-\delta-\eta$.
Any level-$\alpha$ test for the fixed proposal in item 3, powered at least
$1-\eta$ in its beneficial world, needs $\Omega(\ell^{-2r})$ total records
when $\alpha+\eta<1$ are fixed. Uniform source estimation to tolerance below
$\tau$ and failure probability below $1/4$ needs $\Omega(\ell^{-2q})$ records,
even with $s_0$ given. No source upper rate follows from these assumptions alone.
\end{proposition}
\begin{proof}
The sign of the empirical mean of $F_\ell$ makes an orientation error with
probability at most $\exp(-na_D^2\ell^{2r}/2)$. Hence
$n_D\geq2a_D^{-2}\ell^{-2r}\log(1/\delta)$ suffices.
On independent audit records use the selected score and accept when
\begin{equation*}
 \bar V_{\hat s,\ell}-\sqrt{2\log(1/\alpha)/n_C}>0.
\end{equation*}
Conditional on training, item 2 and one-sided Hoeffding give type-I control,
including when learning selected the wrong policy. Conditional on correct
learning, another Hoeffding bound gives power at least $1-\eta$ if
\begin{equation*}
 n_C>\frac{2}{a_C^2}\ell^{-2r}
       \bigl(\sqrt{\log(1/\alpha)}+\sqrt{\log(1/\eta)}\bigr)^2.
\end{equation*}
Integrating over training proves the joint claim.

For the decision-hard pair, product KL and Pinsker give total variation at
most $1/2$ when $n\leq(2C_D)^{-1}\ell^{-2r}$. Le Cam then gives worst-case
orientation error at least $1/4$. The risk decoder transfers this to learning.
For a fixed-candidate test $T\in[0,1]$ satisfying the asserted size and power,
\begin{equation*}
 1-\alpha-\eta\leq\E_+T-\E_-T
 \leq\TV(P_+^n,P_-^n)\leq\sqrt{nC_D\ell^{2r}/2}.
\end{equation*}
Thus $n\geq2(1-\alpha-\eta)^2C_D^{-1}\ell^{-2r}$.
The same testing argument for item 4, followed by midpoint decoding of the
source estimate, gives worst-case source error at least $1/4$ when
$n\leq(2C_S)^{-1}\ell^{-2q}$. These bounds apply to the full record experiment.
Unobserved or additional side information cannot be silently supplied.
\end{proof}

\subsection{A Sufficient Source Upper-Bound Condition}
Suppose additionally that a known $T_\ell(O)\in[-1,1]$ has mean
$\ell^q\phi(a)$, where $\phi$ maps the compact source interval bijectively onto
an interval and $\phi^{-1}$ is $L$-Lipschitz. Project
$\bar T/\ell^q$ onto that image interval and apply $\phi^{-1}$.
Projection is nonexpansive, so
\begin{equation*}
 \Prob(|\hat a-a|\geq\tau)
 \leq2\exp[-n\tau^2\ell^{2q}/(2L^2)].
\end{equation*}
This supplies the matching $O(\ell^{-2q})$ rate. It is a separate quantitative
invertibility assumption, not a consequence of the source-hard pair.

\subsection{KL Bounds from Density Expansions}
Let the dominated densities, relative to a probability measure $\mu$, have
\begin{equation*}
 p_{s,c,\ell}=p_0+\ell^r A_s+\ell^q B_{s,c,\ell},\qquad
 p_{s,c,\ell}\geq\kappa>0\quad\mu\text{-a.e.}
\end{equation*}
Assume $\sup_s\|A_s\|_2\leq M_A$ and
$\sup_{s,c,\ell}\|B_{s,c,\ell}\|_2\leq M_B$.
For two laws sharing $s$, the squared $L^2(\mu)$ difference is at most
$4M_B^2\ell^{2q}$. For different $s$ it is at most
$4(M_A+M_B)^2\ell^{2r}$. The inequality
\begin{equation*}
 \KL(P,Q)\leq\int\frac{(p-q)^2}{q}\,d\mu
 \leq\kappa^{-1}\|p-q\|_2^2
\end{equation*}
therefore supplies items 3--4's distance bounds. Source separation,
Bayes-risk decoding and sign-valid score margins must still be checked.
The density floor can be replaced by direct control of the weighted integral.
It is sufficient, not necessary.

\subsection{Insufficiency of Signal Order}
Consider $P_\ell=\operatorname{Bernoulli}(\ell^2)$ and
$Q_\ell=\operatorname{Bernoulli}(2\ell^2)$. Their cell differences are order
$\ell^2$, but
\begin{equation*}
 \KL(P_\ell,Q_\ell)=(1-\log2)\ell^2+O(\ell^4),
\end{equation*}
not $O(\ell^4)$. A relative-error Chernoff test distinguishes them using
$O(\ell^{-2})$ samples. Pinsker and this expansion give the matching lower
rate at fixed testing errors. The vanishing rare-event probability changes
the noise scale. Also, source laws that are identical for every $c$ satisfy
an $O(\ell^{2q})$ KL bound but permit no consistent source estimator.
Finally, identical observed laws with opposite clean policy gains defeat
powered certification regardless of any observed decision moment. The
clean/observed link in item 2 cannot be inferred from signal order.

\subsection{Continuous-Source Nous HMM as an Instantiation}
For Appendix~\ref{app:shared}'s full law, use $r=1,q=2$ and
\begin{equation*}
 F(O)=\chi(X_1)Y,\quad
 V_t(O)=\ind\{t\chi(X_1)=Y\}-\ind\{X_{1,1}=Y\},\quad
 T(O)=X_{0,1}Y.
\end{equation*}
For every $c\in[1,3/2]$, the exact identities are
\begin{equation*}
 \begin{aligned}
 \E F&=3s\ell/32,& \Delta_t&=\frac c2(tsB-A),\\
 \E V_t&=\frac{\ell A}{2}(tsB-A),&
 \E T&=\frac{\ell^2 A^2}{c^2},\qquad A=1/4,\ B=3/8.
 \end{aligned}
\end{equation*}
In particular $\E V_t=(\ell A/c)\Delta_t$ supplies sign-validity and
$a_C=1/64$. The existing decoder and full-law bounds supply the hard pairs. The constants
$a_D=3/32$, $C_D=9/32$, $C_S=325/41472$ suffice. The clipped inverse moment
in Appendix~\ref{app:shared} is uniform over a continuum, rather than a
two-point parameter classifier. Both persistence and source informativeness
remain unknown. Fixing persistence would destroy this particular source
obstruction. This generalization does not establish the conditions for an
arbitrary real memory source or make the compatible-channel gate necessary
on the symmetric example.

\section{Sparse Revisions: A Disagreement-Conditioned Audit}
\label{app:disagreement}
The global bound in Theorem~\ref{thm:tv} can lose power unnecessarily when
policies agree almost everywhere. Write $D(H)=\ind\{g(H)\ne b(H)\}$,
$d=\Prob(D=1)$ and $\Delta_D=\E[\ind\{g=Z\}-\ind\{b=Z\}\mid D=1]$.
Then $\Delta=d\Delta_D$. In this analysis the rule, policies, witnesses and budget are
fixed before an iid calibration sample. Conditioning on disagreement must not
involve the calibration labels. The reference-channel and TV assumptions of
Theorem~\ref{thm:tv} are unchanged.

\begin{theorem}[Disagreement-conditioned receipt]\label{thm:disagreement}
Let $N=\sum_{i=1}^n D(H_i)$. Construct witness upper bounds $U_k$ on the
full calibration sample at one-sided tail $\alpha/(3K)$, and set
$V_k=\min(1,U_k+\epsilon)$. Let $d_L$ be the one-sided Clopper--Pearson
lower bound on $d$ from $(N,n)$ at tail $\alpha/3$.
For $N>0$, take the report gain $\widehat\Delta_{Y,D}$ and prediction-mass
changes $\widehat m_{D,k}$ only over records with $D=1$, and define
\begin{equation*}
 L_D=\widehat\Delta_{Y,D}-\sum_k(\widehat m_{D,k})_+V_k
       -2\epsilon-\sqrt{\frac{8\log(3/\alpha)}{N}}.
\end{equation*}
Accept iff $N>0$ and $L_D>0$, reporting the positive clean-gain bound
$d_LL_D$. With probability at least $1-\alpha$, every acceptance by this
single test satisfies $0<d_LL_D\leq\Delta$. For $N=0$, decline without
reporting a positive lower bound.
\end{theorem}

\begin{proof}
Witness coverage gives $u\leq V$ simultaneously except on an event of
probability $\alpha/3$. Conditional on $N>0$, the selected observations are
iid from the law given $D=1$. For the fixed true $u$, the scores
$f_H(Y)-u_g+u_b-2\epsilon$ have mean at most $\rho\Delta_D$ and range
$[-2-2\epsilon,2-2\epsilon]$, of width four. Conditional Hoeffding and
integration over $N$ bound the probability that their empirical mean minus
the displayed radius exceeds $\rho\Delta_D$ by $\alpha/3$.
On witness coverage, the box-minimized score is no larger than this true-$u$
score. This is a union bound, not conditioning on estimated $V$.
The binomial event $d_L\leq d$ fails with probability at most $\alpha/3$.
On all three events, acceptance implies $\Delta_D\geq L_D/\rho\geq L_D>0$,
so $\Delta=d\Delta_D\geq d_LL_D>0$. No independence between these three
events is needed. If $d=0$, then $N=0$ almost surely.
\end{proof}

The rule is chosen before calibration. Taking the maximum of this test and
the earlier full-$\alpha$ test after seeing the sample is not covered. Spend
error across tests if both can authorize deployment. This result improves a
conservative implementation, not the sharp population identified set.
Agreement-region cancellation is already explicit in \citet{guerdan}, and
conditioning on discordant outcomes has classical paired-testing precedent
\citep{mcnemar}. The extension in this analysis retains unknown-channel witness coverage,
bounded history dependence and an unconditional positive-gain receipt.

\begin{proposition}[Power and a matching sparse-family obstruction]
\label{prop:sparsepower}
Suppose $m_{D,k}=0$ for every $k$ and
$\mu=\Delta_{Y,D}-2\epsilon>0$. For $\eta\in(0,1)$, the receipt accepts
with probability at least $1-\eta$ if
\begin{equation*}
 nd\geq8\log(2/\eta),\qquad
 nd>\frac{\left[2(K+1)\sqrt{\log(4(K+1)/\eta)}
                         +4\sqrt{\log(3/\alpha)}\right]^2}{\mu^2}.
\end{equation*}
Conversely, consider a balanced binary identity-channel family with
conditional clean gains $\pm\mu$, $0<\mu\leq1/2$.
Any fixed-candidate test of level $\alpha$ and power $1-\eta$ with
$\alpha+\eta<1$ needs
$n\geq(1-\alpha-\eta)^2/(2d\mu^2)$ full records.
Thus the rate is $\Theta((d\mu^2)^{-1})$ for fixed confidence and class count
in this subfamily, not a universal rate over all compatible channels.
\end{proposition}

\begin{proof}
Chernoff gives $N\geq nd/2$ except with probability $\eta/2$ under the first
condition. Conditional on $N$, two-sided Hoeffding for the proxy and the $K$
mass changes gives simultaneous errors at most
$t_N=\sqrt{2\log(4(K+1)/\eta)/N}$ except with probability $\eta/2$.
Since $m_D=0$ and $V_k\leq1$, the empirical box penalty is at most $Kt_N$.
Consequently $L_D\geq\mu-(K+1)t_N-\sqrt{8\log(3/\alpha)/N}>0$.
No useful-witness existence assumption is required for this balanced case.
It remains important for general one-direction revisions.

For the lower bound, let $D\sim\mathrm{Bernoulli}(d)$ and, on disagreement,
let $X$ be an independent fair bit with $g=X,b=1-X$. In the two worlds,
$\Prob(Z=g\mid D,X)=(1\pm\mu)/2$. Off disagreement set $g=b=Z=0$.
Observe the full history and $Y=Z$. Both worlds have $m_D=0$ and gains
$\pm d\mu$. Their one-record KL is
$d\mu\log((1+\mu)/(1-\mu))\leq4d\mu^2$.
The test's acceptance probabilities differ by at least $1-\alpha-\eta$.
Pinsker and product KL give the stated lower bound. For the original global
Hoeffding rule, population signal is $d\mu$ but its radius remains of order
$n^{-1/2}$, explaining the sufficient $d^{-2}\mu^{-2}$ scaling.
\end{proof}

\subsection{Pre-Specified Stress Protocol}
The stress protocol varied class count $K\in\{2,3,5\}$ and calibration size
$n\in\{2000,8000,32000\}$. Disagreement probability was
$d\in\{1,.25,.05,.01\}$. Conditional clean gain was
$\delta\in\{-.6,0,.6\}$, and conditional copying probability was
$\beta\in\{0,.05\}$. We performed 50 repetitions per cell.
Given history $(D,X)$ with $X$ uniform on $K$ classes, $g=X$ and
$b=(X+1)\bmod K$ on disagreement, otherwise $b=X$. Truth equals $g$ with
probability $(1+\delta)/2$ on disagreement and otherwise equals $b$.
Off disagreement truth equals $X$. Thus the exact population gain is $d\delta$.
The base report has $\rho=.8$ and background class-zero probability $.6$.
Copying overwrites reports by $g$ on disagreement with probability $\beta$.
The valid sensitivity budget is $\epsilon=\beta$, justified by construction,
not estimated. Fixed witnesses are $A_k=\{X\ne k\}$.
Each method is evaluated separately at $\alpha=.05$. Their acceptances are
not combined into a deployment rule. The complete grid, seeds, both test
statistics and positive lower bounds are retained.

\subsection{Complete-Grid Outcome}
Table~\ref{tab:acceptance-summary} reports aggregate acceptance counts.
The conditioned rule accepts more beneficial revisions than the global rule.
Both reject every non-improving revision.
Neither produces an observed accepted lower bound above the exact population
gain.
Table~\ref{tab:sparse} reports every sample-size/disagreement combination,
aggregating the six prescribed class-count/copy cells descriptively.
The raw artifact retains every cell, paired repetition and per-cell exact
binomial interval. Table~\ref{tab:replay-summary} gives the replay count.
Acceptance increases with calibration size in the rare-disagreement rows of
Table~\ref{tab:sparse}. The global rule rejects every proposal in those rows.
The design has balanced conditional prediction masses and known synthetic budgets.

\begin{table}[ht]
\centering
\begin{tabular}{rrrr}
\toprule
$n$ & $d$ & Global bound & Conditioned bound\\
\midrule
2000 & 1 & 300 & 300\\
2000 & .25 & 58 & 300\\
2000 & .05 & 0 & 21\\
2000 & .01 & 0 & 0\\
8000 & 1 & 300 & 300\\
8000 & .25 & 300 & 300\\
8000 & .05 & 0 & 300\\
8000 & .01 & 0 & 7\\
32000 & 1 & 300 & 300\\
32000 & .25 & 300 & 300\\
32000 & .05 & 0 & 300\\
32000 & .01 & 0 & 292\\
\bottomrule
\end{tabular}
\caption{Beneficial acceptance counts by calibration size $n$ and disagreement probability $d$.}
\label{tab:sparse}
\end{table}

\section{Proofs and Boundaries for Balanced Batches}\label{app:batch}
All statements condition on the pre-audit information containing the fixed
histories, latent truths, trained scores and proposed revisions. Latent truths
are conditioned on for analysis, not supplied to the procedure.

\subsection{Conservation and Circulation}
Under the exact channel,
\begin{equation*}
 \E W_i=\rho(\ind\{g_i=z_i\}-\ind\{b_i=z_i\})+u_{g_i}-u_{b_i}.
\end{equation*}
Summing over selected records proves the affine identity. For necessity of
universal background cancellation at fixed $\rho\in(0,1)$, take
$u=(1-\rho)e_k$ for every $k$. The resulting equalities force $a_k=0$.
This characterizes cancellation for every background, not all possible cases
of identified gain. The directed incidence matrix has columns
$e_{g_i}-e_{b_i}$. Conservation is its nullspace constraint. Every nonnegative
circulation decomposes into directed cycles, giving the existence criterion.
The incidence matrix is totally unimodular, so the unit-capacity circulation
polytope has integral vertices and a linear score has an integral optimum
\citep{schrijver}. No new network-flow algorithm is asserted.

For exact audit scores $t_i=\E W_i$ and every feasible $x$, cancellation gives
$\sum_i x_it_i=\rho N G_x$. Its ordering over the feasible set is therefore
independent of the unknown background and positive scale. For approximate
scores $s_i$, let $x^*$ maximize $t$ and $\hat x$ maximize $s$. Then
\begin{equation*}
 \rho N(G_{x^*}-G_{\hat x})
 =\sum_i t_i(x_i^*-\hat x_i)
 \leq\sum_i(t_i-s_i)(x_i^*-\hat x_i)
 \leq\sum_i|t_i-s_i|.
\end{equation*}
The last step uses $|x_i^*-\hat x_i|\leq1$. A uniform score error at most
$a_0$ thus gives regret at most $a_0/\rho$. This compares with the best
balanced subset, not the unconstrained best revision.

\subsection{Conditional Receipt}
Enumerate the selected records in a fixed pre-audit order. Let
$\mathcal F_{i-1}$ contain pre-audit information and preceding reports, and
$\zeta_i=\ind\{g_i=z_i\}-\ind\{b_i=z_i\}$.
The contrast has oscillation two. The conditional TV contract gives
\begin{equation*}
 \rho\zeta_i+u_{g_i}-u_{b_i}-2\epsilon
 \leq\E[W_i\mid\mathcal F_{i-1}]
 \leq\rho\zeta_i+u_{g_i}-u_{b_i}+2\epsilon.
\end{equation*}
Each martingale difference has conditional range width two. Conditional
Hoeffding's lemma and the exponential Markov inequality imply
\begin{equation*}
 \Prob\!\left\{\sum_{i=1}^{M}
   (W_i-\E[W_i\mid\mathcal F_{i-1}])>t\right\}
 \leq\exp[-t^2/(2M)].
\end{equation*}
The conditional means sum to at most $\rho N G_x+2M\epsilon$ by exact
conservation. Set $t=\sqrt{2M\log(1/\alpha)}$. Except on an event of
probability $\alpha$, $L_x\leq\rho N G_x/M$. If $L_x>0$, positivity of
$\rho$ gives $G_x>0$ and $\rho\leq1$ yields $(M/N)L_x\leq G_x$.
The argument holds for each pre-audit batch and selection. Integrating over
their randomness preserves coverage without any iid history assumption.

For arbitrary selection, $p=\max_k a_k/M\geq0$ because $\sum_k a_k=0$.
Since $u\geq0$ and $\sum u=1-\rho$, subtracting $p$ instead of cancelling
$u^\top a/M$ produces a lower statistic bounded by $\rho N G_x/M$.
This proves the unbalanced simplex comparator. It does not use a histogram
estimated from an iid sample: the batch counts are fixed and exactly observed.

\subsection{KL--Chernoff Version}
Put $X_i=(W_i+1)/2\in[0,1]$ and
$q_i=\E[X_i\mid\mathcal F_{i-1}]$. For any fixed $\lambda>0$, convexity
gives the conditional moment-generating function bound $1+q_i(e^\lambda-1)$.
If $M^{-1}\sum q_i\leq\theta$ on every path, the exponential
supermartingale and concavity of $q\mapsto\log(1+q(e^\lambda-1))$ give
\begin{equation*}
 \Prob\{\bar X\geq v\}\leq
 \inf_{\lambda>0}\exp\{M\log(1+\theta(e^\lambda-1))-\lambda Mv\}
 =e^{-M\operatorname{kl}(v,\theta)},\quad v>\theta.
\end{equation*}
In this analysis $\theta=(1+\rho N G_x/M+u^\top a/M+2\epsilon)/2$ is a deterministic
upper bound after conditioning on the batch. If $\theta\geq1$, the needed
comparison is trivial. Otherwise invert the tail to obtain a lower endpoint
$l\leq\bar X$ solving $M\operatorname{kl}(\bar X,l)=\log(1/\alpha)$,
with $l=0$ when $\bar X=0$. Replace $\bar W$ minus its Hoeffding radius by
$2l-1$, and apply the same bias corrections. This is bounded-variable Chernoff
inversion, not an exact binomial interval or a new concentration inequality.
Pinsker's inequality implies that its lower endpoint is no smaller than the
Hoeffding endpoint. Each rule is selected before reports. The experiment
does not choose between tests after seeing their results.

\subsection{Power under Misspecification}
For a balanced selection, a worst-case lower bound on the corrected mean is
$\mu=\rho N G_x/M-4\epsilon$: the actual audit can shift its mean down by
$2\epsilon$, and the test subtracts another $2\epsilon$. If $\mu>0$, the
Hoeffding rule accepts with probability at least $1-\eta$ whenever
\begin{equation*}
 M>\frac{2[\sqrt{\log(1/\alpha)}+\sqrt{\log(1/\eta)}]^2}{\mu^2}.
\end{equation*}
For exact channels at zero budget, $\mu=\rho N G_x/M$.
The KL rule inherits this sufficient condition. No uniform useful-batch
existence condition follows: the revision graph can be acyclic.

\subsection{Two Failure Boundaries}
First, a marginal TV guarantee does not imply the conditional contract.
Take a balanced null-gain batch. With shared probability $\beta=.2$, copy
the whole candidate vector. Otherwise report truth. Each marginal kernel is
within $.2$ TV of the identity. For sufficiently large $M$, however, the
copied outcome has $\bar W=1$ and passes even after subtracting $.4$, giving
false acceptance probability $.2$ at nominal $.05$. Conditioning on earlier
reports exposes the shared event and violates the required conditional bound.
An exact Markov-report enumeration provides a separate regression test where
the conditional budget does hold.

Second, distinct unknown signal scales cannot be freely pooled. Two equally
weighted balanced groups with clean gains $.1,-.5$ and scales $.9,.01$ have
positive summed proxy gain but negative summed clean gain. Separate group
receipts can allocate error and commit accepted groups. A pooled receipt
requires a common reference scale or additional information.

\section{Human-Written Dialogue Evaluation}\label{app:multiwoz}
We describe text-model construction, batch selection and the simulated audit protocol.
\subsection{Data and Features}
MultiWOZ contains human task-oriented conversations. Version 2.4 corrects
development and test annotations while retaining version-2.1 training labels
\citep{multiwoz,multiwoz24}. We used the official splits in
Table~\ref{tab:protocols}. The categorical slots were hotel area, price, parking
and internet, restaurant area and price, and attraction area.
The none and dontcare labels were explicit classes.
We excluded restaurant-price annotations outside the prespecified ontology.
Table~\ref{tab:replay-summary} records the turn and exclusion counts.

Local features contained the previous system utterance and current user utterance.
History features added the dialogue prefix through the current user turn.
We excluded current and future system responses, goals, dialogue acts and state metadata.
Term frequency--inverse document frequency (TF--IDF) used word unigrams and
bigrams, minimum document frequency three, and sublinear term frequency.
Table~\ref{tab:text-settings} gives the feature and optimization settings.
We fitted logistic-loss classifiers using stochastic gradient descent (SGD).
All feature mappings and classifiers were fitted on training data.
Development accuracy selected a policy from the prespecified candidate set
with deterministic tie breaking. The history classifier was selected for each slot.
We fixed candidates and models before test evaluation.

\subsection{Memory Policies}
Native Nous, last-report, frequency and recency-weighted aggregators received
the same local reports. Native Nous used the Bayesian updater with default pruning.
The history classifier used the broader contextual features.
Oracle-extraction controls supplied published current-state labels to the
aggregators only for the diagnostic comparison.
Table~\ref{tab:dialogue-baselines} reports all-turn accuracy.
Table~\ref{tab:dialogue-slices} reports specified-value and state-change accuracy.
Table~\ref{tab:dialogue-oracle} reports oracle-extraction accuracy.

\begin{table}[ht]
\centering\small
\begin{tabular}{lrrrrrr}
\toprule
Slot & Nous & Last & Frequency & Recency .5 & Recency .8 & History\\
\midrule
Hotel area & 84.56 & 83.78 & 83.41 & 84.01 & 83.41 & 89.64\\
Hotel price & 83.19 & 82.76 & 81.17 & 82.39 & 81.31 & 89.91\\
Hotel parking & 88.28 & 88.01 & 86.46 & 88.23 & 87.38 & 95.40\\
Hotel internet & 87.53 & 87.57 & 86.07 & 87.75 & 86.91 & 95.01\\
Restaurant area & 81.77 & 79.96 & 79.63 & 80.53 & 79.94 & 90.75\\
Restaurant price & 82.55 & 80.41 & 79.58 & 81.48 & 80.39 & 92.45\\
Attraction area & 84.79 & 83.40 & 81.20 & 83.27 & 81.80 & 90.88\\
\bottomrule
\end{tabular}

\caption{All-turn slot accuracy (\%) by memory policy.}
\label{tab:dialogue-baselines}
\end{table}
\begin{table}[ht]
\centering\small
\begin{tabular}{lrrrr}
\toprule
Slot & \multicolumn{2}{c}{Specified values} & \multicolumn{2}{c}{State changes}\\
\midrule
 & Nous & History & Nous & History\\
Hotel area & 22.83 & 52.52 & 27.75 & 40.53\\
Hotel price & 21.68 & 52.90 & 29.77 & 48.09\\
Hotel parking & 28.80 & 73.07 & 61.11 & 65.00\\
Hotel internet & 23.75 & 70.55 & 54.26 & 71.28\\
Restaurant area & 39.61 & 72.17 & 53.73 & 60.90\\
Restaurant price & 40.59 & 76.99 & 53.38 & 67.85\\
Attraction area & 38.44 & 66.59 & 45.91 & 46.98\\
\bottomrule
\end{tabular}

\caption{Slot accuracy (\%) for specified values and state-change turns.}
\label{tab:dialogue-slices}
\end{table}
\begin{table}[ht]
\centering\small
\begin{tabular}{lrrrrr}
\toprule
Slot & Nous & Last & Frequency & Recency .5 & Recency .8\\
\midrule
Hotel area & 99.19 & 100.00 & 93.92 & 100.00 & 97.42\\
Hotel price & 98.93 & 100.00 & 93.42 & 100.00 & 97.44\\
Hotel parking & 99.24 & 100.00 & 94.74 & 100.00 & 97.73\\
Hotel internet & 99.19 & 100.00 & 94.79 & 100.00 & 97.79\\
Restaurant area & 99.27 & 100.00 & 93.45 & 100.00 & 97.25\\
Restaurant price & 99.09 & 100.00 & 93.91 & 100.00 & 97.74\\
Attraction area & 99.44 & 100.00 & 93.03 & 100.00 & 97.10\\
\bottomrule
\end{tabular}

\caption{All-turn slot accuracy (\%) with oracle extraction.}
\label{tab:dialogue-oracle}
\end{table}
\begin{table}[ht]
\centering\small
\begin{tabular}{lrr}
\toprule
Policy & Brier & Log loss\\
\midrule
native nous & 0.2574 & 1.9914\\
last report & 0.2496 & 0.5308\\
frequency & 0.2273 & 0.8931\\
recency 0.5 & 0.2617 & 0.5554\\
recency 0.8 & 0.3807 & 0.7987\\
history classifier & 0.1137 & 0.2266\\
\bottomrule
\end{tabular}

\caption{Mean Brier score and log loss across categorical slots.}
\label{tab:dialogue-probabilities}
\end{table}

Table~\ref{tab:dialogue-probabilities} evaluates raw predictive probabilities.
The Brier score sums squared errors over classes.
We clipped probabilities at $10^{-12}$ when computing log loss.

\subsection{Batch Construction}
Each slot supplied a fixed batch of valid test turns, including dependent turns
from the same dialogue. Proposal edges changed native Nous predictions to
candidate predictions. Scores were the history classifier's difference between
candidate-label and incumbent-label probabilities. A count-only circulation
assigned unit scores. We selected subsets before generating audit reports.
Table~\ref{tab:retained-gain} compares revision retention and clean gain.

\subsection{Audit Generation}
For each slot, we used $\rho\in\{.3,.6,.9\}$ and copying probability
$\beta\in\{0,.03,.1\}$. We generated fresh audit vectors from
$C_0=\rho I+(1-\rho)v\mathbf1^\top$, where $v_{\mathrm{none}}=.8$ and
the remaining background mass was uniform.
We independently replaced each report with the candidate label with
probability $\beta$. The conditional deviation budget was $\epsilon=\beta$.
Forward and reversed comparisons used the same selected masks.
Reversal swapped the predictions and directed copying toward the reversed candidate.
The reversed gain was negative for every retained subset.

All methods shared the underlying reports within an audit draw.
The scored balanced and full corrected methods evaluated different subsets.
Count-only balancing isolated the effect of learned scores.
The known-channel oracle subtracted the true $u$ offset and divided by $\rho$.
The proposed methods did not receive these parameters.
The naive method omitted channel-bias and copying corrections.
Each method used KL--Chernoff inversion at $\alpha=.05/7$ per slot.
We evaluated each method, orientation and channel setting separately.

\begin{table}[ht]
\centering\small
\begin{tabular}{rrrrrrr}
\toprule
$\rho$ & $\epsilon$ & Balanced & Count-only & Full corrected & Oracle & Naive harmful\\
\midrule
0.3 & 0.00 & 241 & 98 & 0 & 691 & 582\\
0.3 & 0.03 & 102 & 31 & 0 & 642 & 677\\
0.3 & 0.10 & 5 & 0 & 0 & 328 & 700\\
0.6 & 0.00 & 688 & 629 & 0 & 700 & 0\\
0.6 & 0.03 & 656 & 523 & 0 & 700 & 0\\
0.6 & 0.10 & 404 & 192 & 0 & 699 & 0\\
0.9 & 0.00 & 700 & 698 & 500 & 700 & 0\\
0.9 & 0.03 & 700 & 690 & 490 & 700 & 0\\
0.9 & 0.10 & 700 & 573 & 86 & 700 & 0\\
\bottomrule
\end{tabular}

\caption{Forward acceptance counts and naive harmful reverse acceptances.
Rows pool seven slots and 100 draws per slot.}
\label{tab:dialogue-audit-grid}
\end{table}

Table~\ref{tab:dialogue-audit-grid} gives the channel grid.
The corrected methods rejected all harmful reverse proposals.
The scored balanced column has greater acceptance than the full corrected
column at lower report signal. Table~\ref{tab:deployed-gain} compares their utility.

\subsection{Replay and Uncertainty}
We reconstructed audit reports from saved seeds.
A separate scalar-root implementation checked each confidence bound.
Hash checks verified the data set, models, predictions and results.
Table~\ref{tab:replay-summary} records the replay quantities.
We computed predictive-gain intervals using dialogue-cluster bootstrap resampling.
Each slot's interval excluded zero.
These intervals compare the fixed candidate with native Nous.

\section{Additional Evaluation Settings and Summaries}\label{app:style-settings}
Tables~\ref{tab:study-settings} and \ref{tab:replay-summary} give the evaluation
settings and replay quantities. The paired gain intervals used the normal
approximation and were descriptive across comparisons.

\begin{table}[ht]
\centering\small
\begin{tabular}{ll}
\toprule
Setting & Value\\
\midrule
Mutable-state runs & 9\\
Mutable-state snapshots and sources per snapshot & 4 and 5\\
Mutable-state copy probability & $.9$\\
Native Nous declared reliability & $.75$\\
HMM expectation-maximization iterations and restarts & 25 and 3\\
Mutable-state audit $(q_0,q_1)$ & $(.45,.95)$\\
MiniGrid source accuracies & $(.85,.60,.70,.65,.70)$\\
MiniGrid copy probability & $.9$\\
MiniGrid audit $(q_0,q_1)$ & $(.1,.9)$\\
MiniGrid runs per environment & 3\\
Initial categorical classes and probe lags & $\{2,3,5\}$ and $\{1,3,6\}$\\
Initial categorical calibration sizes & $\{2000,8000\}$\\
Initial categorical repetitions per cell & 100\\
Repeated-release series and releases & 200 and 10\\
Phase report signals & $\{.2,.4,.6,.8,.95\}$\\
Phase background class-zero masses & $\{.25,.5,.75\}$\\
Phase calibration size and repetitions per cell & 32000 and 100\\
Sparse-revision cells and repetitions per cell & 216 and 50\\
\bottomrule
\end{tabular}
\caption{Evaluation settings. HMM denotes hidden Markov model.}
\label{tab:study-settings}
\end{table}

\begin{table}[ht]
\centering\small
\begin{tabular}{lr}
\toprule
Replay quantity & Count or value\\
\midrule
Mutable-state held-out queries & 45,000\\
Selected HMM and suffix-two runs & 7 and 2\\
Mutable-state split seeds & 36\\
Mutable-state runtime (seconds) & 42.55\\
MiniGrid test episodes & 9,000\\
MiniGrid reconstructed certificates & 36\\
MiniGrid distinct calibration/test seeds & 54,000\\
Initial categorical calibrations & 5,400\\
Phase calibrations & 4,500\\
Sequential-release calibrations & 2,000\\
Sparse-revision paired repetitions & 10,800\\
Learned-witness cells and split seeds & 12 and 48\\
MultiWOZ user turns & 7,372\\
Excluded restaurant-price labels & 12\\
MultiWOZ audit vectors & 12,600\\
MultiWOZ method evaluations and audit cells & 63,000 and 126\\
MultiWOZ dialogue-cluster bootstrap draws & 1,000\\
\bottomrule
\end{tabular}
\caption{Evaluation and replay quantities. Runtime is measured on the local host.}
\label{tab:replay-summary}
\end{table}

Table~\ref{tab:text-settings} gives the text-model settings and audit repetitions.
\begin{table}[ht]
\centering\small
\begin{tabular}{lr}
\toprule
Text-model or audit setting & Value\\
\midrule
Maximum features per view & 35,000\\
Maximum classifier iterations & 50\\
Classifier regularization & $10^{-5}$\\
Classifier convergence tolerance & $10^{-4}$\\
Recency decay factors & $.5,.8$\\
Candidate policies & 6\\
Audit draws per slot and setting & 100\\
\bottomrule
\end{tabular}
\caption{Text-model and audit settings.}
\label{tab:text-settings}
\end{table}

Table~\ref{tab:retained-gain} compares selected subsets with the full proposal.
Table~\ref{tab:deployed-gain} gives deployed gain under no-copy audits.
The corrected proposal has greater deployed gain at high report signal.
The scored balanced proposal has greater deployed gain at the intermediate signal.

\begin{table}[ht]
\centering\small
\begin{tabular}{lrr}
\toprule
Proposal & Retained revisions & Mean clean gain (pp)\\
\midrule
Full candidate & 5,987 & 7.34\\
Scored balanced & 2,236 & 2.26\\
Count-only balanced & 2,243 & 1.77\\
\bottomrule
\end{tabular}
\caption{Retained MultiWOZ revisions and mean clean gains.
The abbreviation pp denotes percentage points.}
\label{tab:retained-gain}
\end{table}

\begin{table}[ht]
\centering\small
\begin{tabular}{lrr}
\toprule
Reference signal & Scored balanced gain (pp) & Full corrected gain (pp)\\
\midrule
$.6$ & 2.23 & 0\\
$.9$ & 2.26 & 5.25\\
\bottomrule
\end{tabular}
\caption{Mean deployed gain under no-copy audits.
The abbreviation pp denotes percentage points.}
\label{tab:deployed-gain}
\end{table}

Table~\ref{tab:uci-diagnostics} retains the previously exposed diagnostics
\citep{eeg,occupancy}. The nominal test accepted a harmful electroencephalography
(EEG) revision. The unrestricted-channel certificate rejected that revision
and also rejected the beneficial occupancy revision.

\begin{table}[ht]
\centering\small
\begin{tabular}{lrrl}
\toprule
Dataset & Incumbent accuracy (\%) & Candidate accuracy (\%) & Outcome\\
\midrule
EEG & 70.69 & 37.07 & Harmful acceptance\\
Occupancy & 55.36 & 99.49 & Beneficial rejection\\
\bottomrule
\end{tabular}
\caption{Previously exposed time-series diagnostics.
EEG denotes electroencephalography.}
\label{tab:uci-diagnostics}
\end{table}

Table~\ref{tab:arithmetic-check} summarizes the conditional-binomial calculation.
\begin{table}[ht]
\centering\small
\begin{tabular}{lr}
\toprule
Calculation quantity & Value\\
\midrule
Specified laws & 27\\
Null or harmful laws & 18\\
Largest null or harmful acceptance probability & $1.16\times10^{-21}$\\
\bottomrule
\end{tabular}
\caption{Conditional-binomial acceptance calculation.}
\label{tab:arithmetic-check}
\end{table}

\fi

\end{document}